\documentclass[a4paper,11pt]{article}

\usepackage[margin=1.2in]{geometry}

\usepackage{amsmath}
\usepackage{algorithmic}
\usepackage{algorithm}
\usepackage{amsthm}
\usepackage{amsfonts}
\usepackage{comment}
\usepackage{graphicx}
\usepackage{color}
\usepackage{subcaption}

\newtheorem{remark} {Remark}
\newtheorem{theorem} {Theorem}
\newtheorem{lemma} {Lemma}
\newtheorem{definition} {Definition}

\newtheorem{assumption} {Assumption}
\newtheorem{proposition} {Proposition}

\def\x{{\mathbf{x}}}

\def\z{{\mathbf{z}}}

\def\y{{\mathbf{y}}}

\def\X{{\mathbf{X}}}
\def\Y{{\mathbf{Y}}}

\def\A{{\mathbf{A}}}
\def\M{{\mathbf{M}}}
\def\N{{\mathbf{N}}}
\def\I{{\mathbf{I}}}

\def\V{{\mathbf{V}}}
\def\Z{{\mathbf{Z}}}
\def\W{{\mathbf{W}}}
\def\U{{\mathbf{U}}}
\def\Q{{\mathbf{Q}}}
\def\P{{\mathbf{P}}}

\def\L{{\mathbf{L}}}
\def\S{{\mathbf{S}}}

\def\nucnorm{{\textrm{nuc}}}

\newcommand{\vecspace}{\textbf{E}}

\DeclareMathOperator*{\argmin}{arg\,min}

\newcommand{\mY}{\mathcal{Y}}

\newcommand{\mA}{\mathcal{A}}

\newcommand{\mK}{\mathcal{K}}

\newcommand{\mX}{\mathcal{X}}

\newcommand{\mR}{\mathcal{R}}

\newcommand{\trace}{\textrm{Tr}}
\newcommand{\rank}{\textrm{rank}}
\newcommand{\reals}{\mathbb{R}}

\newcommand{\exreals}{\left(-\infty , +\infty\right]} 

\newcommand{\norm}[1]{\left\Vert {#1} \right\Vert} 
\newcommand{\act}[1]{\left\langle {#1} \right\rangle} 
\newcommand{\dom}[1]{\mathrm{dom}\,{#1}} 
\newcommand{\erl}{\left(-\infty , +\infty\right]} 

\title{Improved Complexities of Conditional Gradient-Type Methods with Applications to Robust Matrix Recovery Problems}
\author{Dan Garber \\ {\small \texttt{dangar@technion.ac.il}}
\and
Atara Kaplan \\  {\small \texttt{ataragold@technion.ac.il}}
\and
Shoham Sabach  \\ {\small \texttt{ssabach@technion.ac.il}} \\ \\
Technion - Israel Institute of Technology
}

\date{}
\begin{document}

\maketitle

\begin{abstract}
Motivated by robust matrix recovery problems such as Robust Principal Component Analysis, we consider a general optimization problem of minimizing a smooth and strongly convex loss function applied to the sum of two blocks of variables, where each block of variables is constrained or regularized individually. We study a Conditional Gradient-Type method which is able to leverage the special structure of the problem to obtain faster convergence rates than those attainable via standard methods, under a variety of assumptions. In particular, our method is appealing for matrix problems in which one of the blocks corresponds to a low-rank matrix since it avoids prohibitive full-rank singular value decompositions required by most standard methods. While our initial motivation comes from problems which originated in statistics, our analysis does not impose any statistical assumptions on the data.
\end{abstract}

\section{Introduction}

In this paper we consider the following general convex optimization problem
\begin{equation} \label{GeneralModel}
	\min \left\{ f\left(\X , \Y\right) := g\left(\X + \Y\right) + \mR_{\mX}\left(\X\right) + \mR_{\mY}\left(\Y\right) : \, \X , \Y \in \vecspace \right\},
\end{equation}
where $\vecspace$ is a finite-dimensional normed vector space over the reals, $g : \vecspace \rightarrow \reals$ is assumed to be continuously differentiable and strongly convex, while $\mR_{\mX} : \vecspace \rightarrow \exreals$ and $\mR_{\mY} : \vecspace \rightarrow \exreals$ are proper, lower semicontinuous and convex functions which can be thought of either as regularization functions, or indicator functions\footnote{An indicator function of a set is defined to be $0$ in the set and $+\infty$ outside.} of certain closed and convex feasible sets $\mX$ and $\mY$.
\medskip

Problem \eqref{GeneralModel} captures several important problems of interest, perhaps the most well-studied is that of \textit{Robust Principal Component Analysis (PCA)} \cite{Candes11,wright2009robust,mu2016scalable}, in which the goal is to (approximately) decompose an $m\times n$ input matrix $\M$ into the sum of a low-rank matrix $\X$ and a sparse matrix $\Y$. The underlying optimization problem for Robust PCA can be written as (see for instance \cite{mu2016scalable})
\begin{equation} \label{eq:robustPCA}
	\min \left\{ \frac{1}{2}\norm{\X + \Y - \M}_{F}^{2} : \, \norm{\X}_{\nucnorm} \leq \tau, \norm{\Y}_{1} \leq s, \, \X , \Y \in \reals^{m \times n} \right\},
\end{equation}
where $\norm{\cdot}_{F}$ denotes the Frobenius norm, $\norm{\cdot}_{\textrm{nuc}}$ denotes the nuclear norm, i.e., the sum of singular values, which is a highly popular convex surrogate for low-rank penalty, and $\norm{\cdot}_{1}$ is the entry-wise $\ell_{1}$-norm, which is a well-known convex surrogate for entry-wise sparsity.
\medskip

Other variants of interest of Problem \eqref{eq:robustPCA} are when the data matrix $\M$ is a corrupted covariance matrix, in which case it is reasonable to further constrain $\X$ to be positive semidefinite, i.e., use the constraints $\X \succeq \mathbf{0}$ and $\trace(\X) \leq \tau$. In the case that $\M$ is assumed to have several fully corrupted rows or columns, a popular alternative to the $\ell_{1}$-norm regularizer on the variable $\Y$ is to use either the norm $\norm{\cdot}_{1,2}$ (sum of $\ell_{2}$-norm of rows) in case of corrupted rows, or the norm $\norm{\cdot}_{2,1}$ (sum of $\ell_{2}$-norm of columns) in case of corrupted columns, as a regularizer/constraint \cite{xu2010robust}. Finally, moving beyond Robust PCA, a different choice of interest for the loss $g\left(\cdot\right)$ could be $g\left(\Z\right) := \left(1/2\right)\norm{\mA\Z - \M}_{F}^{2}$, where $\mA$ is a linear sensing operator such that $\mA^{T}\mA$ is positive definite (so $g\left(\cdot\right)$ is strongly convex). 
\medskip

In this paper we present an algorithm and analyses that build on the special structure of Problem \eqref{GeneralModel}, which improve upon state-of-the-art complexity bounds, under several different assumptions. A common key to all of our results is the ability to exploit the strong convexity of $g(\cdot)$ to obtain improved complexity bounds. Here it should be noted that while $g\left(\cdot\right)$ is assumed to be strongly convex, Problem \eqref{GeneralModel} is in general not strongly convex in $\left(\X , \Y\right)$. This can already be observed when choosing $g\left(\z\right) := \frac{1}{2}\norm{\z}_{2}^{2}$, and $\mR_{\mX}\left(\cdot\right) = \mR_{\mY}\left(\cdot\right) = 0$, where $\x , \y \in \reals^{d}$. In this case, denoting the overall objective as $f\left(\x , \y\right) := \frac{1}{2}\norm{\x + \y}_{2}^{2}$, it is easily observed that the Hessian matrix of $f\left(\cdot , \cdot\right)$ is given by $\nabla^{2} f\left(\x , \y\right) = \left(\I~~\I\right)^{\top}\left(\I~~\I\right)$, and hence is not full-rank. 
\medskip

The fastest known convergence rate for first-order methods applicable to Problem \eqref{GeneralModel}, is achievable by accelerated gradient methods such as Nesterov's optimal method \cite{Nesterov13} and FISTA \cite{FISTA}, which converge at a rate of $O(1/t^{2})$. However, in the context of low-rank matrix optimization problems such as Robust PCA, these methods require to compute a full-rank singular value decomposition on each iteration to update the low-rank component, which is often prohibitive for large scale instances. A different type of first-order methods is the Conditional Gradient (CG) Method (a.k.a Frank-Wolfe algorithm) and variants of \cite{GH15,GH16,Gidel2018FrankWolfeSV,pmlr-v54-huang17a,Jaggi13b,pmlr-v80-yurtsever18a}. In the context of low-rank matrix optimization, the CG method simply requires to compute an approximate leading singular vector pair of the negative gradient at each iteration, i.e., a rank-one SVD. Hence, in this case, the CG method is much more scalable, than projection/proximal based methods. However, the rate of convergence is slower, e.g., if both $\mR_{\mX}\left(\cdot\right)$ and $\mR_{\mY}\left(\cdot\right)$ are indicator functions of certain closed and convex sets $\mX$ and $\mY$, then the convergence rate of the conditional gradient method is of the form $O((D_{\mX}^{2} + D_{\mY}^{2})/t)$, where $D_{\mX}$ and $D_{\mY}$ denote the Euclidean diameter of the corresponding feasible sets $\mX$ and $\mY$, where the diameter of a subset $C$ of $\reals^{d}$ is defined by $D_{C} = \max_{\x_{1} , \x_{2} \in C} \norm{\x_{1} - \x_{2}}_{2}$.
\medskip

Recently, two variants of the conditional gradient method for low-rank matrix optimization were suggested, which enjoy faster convergence rates when the optimal solution has low rank (which is indeed a key implicit assumption in such problems), while requiring to compute only a single low-rank SVD on each iteration \cite{Garber16a,allen2017linear}. However, both of these new methods require the objective function to be strongly convex, which as we discussed above, does not hold in our case. Nevertheless, both our algorithm and our analysis are inspired by these two works. In particular, we generalize the low-rank SVD approach of \cite{allen2017linear} to non-strongly-convex problems of the form of Problem \eqref{GeneralModel}, which include arbitrary regularizers or constraints. 
\medskip

In another recent related work \cite{mu2016scalable}, which also serves as a motivation for this current work, the authors considered a variant of the conditional gradient method tailored for low-rank and robust matrix recovery problems such as Problem \eqref{eq:robustPCA}, which combines standard conditional gradient updates of the low-rank variable (i.e., rank-one SVD) and proximal gradient updates for the sparse noisy component. However, both the worst-case convergence rate and running time do not improve over the standard conditional gradient method. Combining conditional-gradient and proximal-gradient updates for low-rank models was also considered in \cite{NIPS2015_5979} for solving a convex optimization problem related to temporal recommendation systems.
\medskip

Finally, it should be noted that while developing efficient \textit{non-convex optimization}-based algorithms for Robust PCA with provable guarantees is an active subject (see e.g., \cite{netrapalli2014non,yi2016fast}), such works fall short in two aspects: (a) they are not flexible as the general model \eqref{GeneralModel}, which allows for instance to impose a PSD constraint on the low-rank component or to consider various sparsity-promoting regularizers for the sparse component $\Y$, and (b) all provable guarantees are heavily based on assumptions on the input matrix $\M$ (such as incoherence of  the singular value decomposition of the low-rank component or assuming certain patterns of the sparse component), which can be quite limiting in practice. This work, on the other hand, is completely free of such assumptions.
\medskip

To overcome the shortcomings of previous methods applicable to Problem \eqref{GeneralModel}, in this paper we present a first-order method, which combines two well-known ideas, for tackling Problem \eqref{GeneralModel}. In particular we show that under several assumptions of interest, despite the fact that the objective in Problem \eqref{GeneralModel} is in general not strongly convex, it is possible to leverage the strong convexity of $g\left(\cdot\right)$ towards obtaining better complexity results, while applying update steps that are scalable to large scale problems. Informally speaking, our main improved complexity bounds are as follows:
\begin{enumerate}
	\item In the case that both $\mR_{\mX}\left(\cdot\right)$ and $\mR_{\mY}\left(\cdot\right)$ are indicators of compact and convex sets (as in Problem \eqref{eq:robustPCA}), we obtain convergence rate of $O(\min\{D_{\mX}^{2} , D_{\mY}^{2}\}/t)$. In particular when $\X$ is constrained, for example, via a low-rank promoting constraint, such as the nuclear-norm, our method requires on each iteration only a SVD computation of rank=$\rank(\X^{\ast})$, where $\X^{\ast}$ is part of certain optimal solution $\left(\X^{\ast} , \Y^{\ast}\right)$. This result improves (in terms of running time), in a wide regime of parameters, mainly when $\min\{ D_{\mX}^{2} , D_{\mY}^{2} \} << \max\{ D_{\mX}^{2} , D_{\mY}^{2} \}$, over the conditional gradient method which converges with rate $O(\max\{D_{\mX}^{2} , D_{\mY}^{2}\}/t)$, and over accelerated gradient methods which require, in the context of low-rank matrix optimization problems, a full-rank SVD computation on each iteration.
	\medskip
	\item In the case that $\mR_{\mY}\left(\cdot\right)$ is an indicator of a strongly convex set (e.g., an $\ell_{p}$-norm ball for $p \in \left(1 , 2\right]$), our method achieves a fast convergence rate of $O(1/t^2)$. As in the previous case, if $\X$ is constrained/regularized via the nuclear norm, then our method only requires a SVD computation of rank=$\rank(\X^{\ast})$. To the best of our knowledge, this is the first result that combines an $O(1/t^{2})$ convergence rate and low-rank SVD computations in this setting. In particular, in the context of Robust PCA, such a result allows us to replace a traditional sparsity-promoting constraint of the form $\norm{\Y}_{1} \leq \tau$ with $\norm{\Y}_{1 + \delta} \leq \tau'$, for some small constant $\delta$. Using the $\ell_{1 + \delta}$-norm instead of the $\ell_{1}$-norm gives rise to a strongly convex feasible set and, as we demonstrate empirically in Section \ref{sec:robustPCAexp}, may provide a satisfactory approximation to the $\ell_{1}$-norm constraint in terms of sparsity. 
	\medskip
	\item In the case that either $\mR_{\mX}\left(\cdot\right)$ or $\mR_{\mY}\left(\cdot\right)$ are strongly convex (though not necessarily differentiable), our method achieves a linear convergence rate. In fact, we show that even if only one of the variables is regularized by a strongly convex function, then the entire objective of Problem \eqref{GeneralModel} becomes strongly convex in $\left(\X , \Y\right)$. Here also, in the case of a nuclear norm constraint/regularization on one of the variables, we are able to leverage the use of only low-rank SVD computations. In the context of Robust PCA such a natural strongly convex regularizer may arise by replacing the $\ell_{1}$-norm regularization on $\Y$ with the elastic net regularizer, which combines both the $\ell_{1}$-norm and the squared $\ell_{2}$-norm, and serves as a popular alternative to the $\ell_{1}$-norm regularizer in LASSO.
\end{enumerate}
A quick summary of the above results in the context of Robust PCA problems, such as Problem \eqref{eq:robustPCA}, is given in Table \ref{table:results}. See Section \ref{sec:robustPCAexp} in the sequel for a detailed discussion.

\begin{table*}\renewcommand{\arraystretch}{1.3}
{\small
\begin{center}{\footnotesize
  \begin{tabular}{| l | c | c | c | c | c | c |} \hline
    & \multicolumn{2}{|c|}{Cond. Grad.\cite{Jaggi13b}}& \multicolumn{2}{|c|}{FISTA \cite{FISTA}} &\multicolumn{2}{|c|}{Algorithm \ref{alg:GCG}} \\ \hline
    setting &  rate & SVD &  rate & SVD & rate & SVD \\ 
    & & rank & & rank & & rank \\\hline 
    $\tau >> s$ (``high SNR regime")& $\tau^{2}/t$ & $1$ & $\tau^{2}/t^{2}$ & $n$ & $s^{2}/t$ & $\rank(\X^{\ast})$ \\ \hline
    $\tau << s$ (``low SNR regime") & $s^{2}/t$ & $1$ & $s^{2}/t^{2}$ & $n$ & $\tau^{2}/t$ & $1$ \\ \hline
    $\mY := \{ \Y : \norm{\Y}_{1 + \delta} \leq s \}$ & $\frac{\max\{s , \tau\}^{2}}{t}$ & $1$ & $\frac{\max\{s , \tau\}^{2}}{t^{2}}$ & $n$ & $\frac{s^{2}n^{2\frac{1 - \delta}{1 + \delta}}}{t^{2}}$ & $\rank(\X^{\ast})$ \\ \hline
    $\mR_{\mY}(\Y) = \lambda_{1}\norm{\Y}_{1} + \lambda_{2}\norm{\Y}_{F}^{2}$ & $1/t$ & $1$ & $e^{-\Theta(\sqrt{\lambda_{2}}t)}$ & $n$ & $e^{-\Theta(\lambda_{2}t)}$ & $\rank(\X^{\ast})$ \\ \hline
  \end{tabular}
  \caption{Comparison of convergence rates and iteration complexity bounds for Robust PCA problems (see Problem \eqref{eq:robustPCA}) with a $n \times n$ input matrix $\M$. For all methods the computational bottleneck is a single SVD computation to update the variable $\X$, hence we focus on the rank of the required SVD. For clarity of presentation the results are given in simplified form. The dependence on $n$ in the rate for Algorithm \ref{alg:GCG} in the third row comes from the strong convexity parameter of the set $\mY$.}
  \label{table:results}}
\end{center}}
\vskip -0.2in
\end{table*}\renewcommand{\arraystretch}{1}

\section{Preliminaries}
Throughout the paper we let $\vecspace$ denote an arbitrary finite-dimensional normed vector space over $\reals$ where $\norm{\cdot}$ and $\norm{\cdot}_{\ast}$ denote the primal and dual norms over $\vecspace$, respectively.

\subsection{Smoothness and strong convexity of functions and sets}

\begin{definition}[smooth function] \label{def:smoothfunc}
	Let $f : \vecspace \rightarrow \reals$ be a continuously differentiable function over a convex set $\mK \subseteq \vecspace$. We say that $f$ is $\beta$-smooth over $\mK$ with respect to $\norm{\cdot}$, if for all $\x , \y \in \mK$ it holds that $f\left(\y\right) \leq f\left(\x\right) + \act{\y - \x , \nabla f\left(\x\right)} + \left(\beta/2\right)\norm{\x - \y}^{2}$.
\end{definition}

\begin{definition}[strongly convex function]\label{def:strongconvexfunc}
	Let $f : \vecspace \rightarrow \reals$ be a continuously differentiable function over a convex set $\mK \subseteq \vecspace$. We say that $f$ is $\alpha$-strongly convex over $\mK$ with respect to $\norm{\cdot}$, if it satisfies for all $\x , \y \in \mK$ that $f\left(\y\right) \geq f\left(\x\right) + \act{\y - \x , \nabla f\left(\x\right)} + \left(\alpha/2\right)\norm{\x - \y}^{2}$.
\end{definition}
The above definition combined with the first-order optimality condition implies that for a continuously differentiable and $\alpha$-strongly convex function $f$, if $\x^{\ast} = \arg\min_{\x \in \mK} f\left(\x\right)$, then for any $\x \in \mK$ it holds that $f\left(\x\right) - f\left(\x^{\ast}\right) \geq \left(\alpha/2\right)\norm{\x - \x^{\ast}}^{2}$.
\medskip

This last inequality further implies that the magnitude of the gradient of $f$ at point $\x$, $\norm{\nabla f\left(\x\right)}_{\ast}$ is at least of the order of the square-root of the objective value approximation error at $\x$, that is, $f\left(\x\right) - f\left(\x^{\ast}\right)$. Indeed, this follows since
\begin{align*}
	\sqrt{\frac{2}{\alpha}\left(f\left(\x\right) - f\left(\x^{\ast}\right)\right)} \cdot \norm{\nabla f\left(\x\right)}_{\ast} & \geq \norm{\x - \x^{\ast}} \cdot \norm{\nabla f\left(\x\right)}_{\ast} \\ 
	& \geq \act{\x - \x^{\ast} , \nabla f\left(\x\right)} \\
 	& \geq f\left(\x\right) - f\left(\x^{\ast}\right),
\end{align*}
where the second inequality follows from Holder's inequality and the third from the convexity of $f$. Thus, at any point $\x \in \mK$, it holds that
\begin{equation} \label{ie:largegrad}
	\norm{\nabla f\left(\x\right)}_{\ast} \geq \sqrt{\frac{\alpha}{2}} \cdot \sqrt{f\left(\x\right) - f\left(\x^{\ast}\right)}.
\end{equation}
\begin{definition}[strongly convex set] \label{def:strongconvexset}
	We say that a convex set $\mK \subset \vecspace$ is $\alpha$-strongly convex with respect to $\norm{\cdot}$ if for any $\x , \y \in \mK$, any $\gamma \in \left[0 , 1\right]$ and any vector $\z \in\vecspace$ such that $\norm{\z} = 1$, it holds that $\gamma\x + \left(1 - \gamma\right)\y + \gamma\left(1 - \gamma\right)\left(\alpha/2\right)\norm{\x - \y}^{2}\z \in \mK$. That is, $\mK$ contains a ball of radius $\gamma\left(1 - \gamma\right)\left(\alpha/2\right)\norm{\x - \y}^{2}$ induced by the norm $\norm{\cdot}$ centered at $\gamma\x + \left(1 - \gamma\right)\y$.
\end{definition}
For more details on strongly convex sets, examples and connections to optimization, we refer the reader to \cite{GH15}.

\section{Algorithm and Results}
As discussed in the introduction, in this paper we study efficient algorithms for the minimization model \eqref{GeneralModel}, where, throughout the paper, our blanket assumption is as follows
\begin{assumption} \label{Assumption1}
\begin{itemize}
\item $g : \vecspace \rightarrow \reals$ is $\beta$-smooth and $\alpha$-strongly convex.
\item $\mR_{\mX} : \vecspace \rightarrow \erl$ and $\mR_{\mY}  : \vecspace \rightarrow \erl$ are proper, lower semicontinuous and convex functions.
\end{itemize}
\end{assumption}
It should be noted that since $\mR_{\mX}\left(\cdot\right)$ (similarly for $\mR_{\mY}\left(\cdot\right)$) is assumed to be extended-valued function, it allows the inclusion of constraint through the indicator function of the corresponding constraint set. Indeed, in this case one will consider $\mR_{\mX}\left(\X\right) := \iota_{\mX}\left(\X\right)$, where $\mX \subset \vecspace$ is a nonempty, closed and convex.
\medskip

We now present the main algorithmic framework, which will be used to derive all of our results.

\begin{algorithm}
	\caption{Alternating Conditional Gradient Proximal Gradient Method}
	\begin{algorithmic}[1] \label{alg:GCG}
		\STATE input: $\left\{ \eta_{t} \right\}_{t \geq 1} \subset \left[0 , 1\right]$ - sequence of step-sizes.
		\STATE $\X_{1}$ is an arbitrary point in $\dom{\mR_{\mX}}$, $\Y_{1}$ is an arbitrary point in $\dom{\mR_{\mY}}$.
		\FOR{$t = 1 , 2 , \ldots$}
		\STATE $\W_{t} = \argmin\limits_{\W \in \vecspace} \left\{ \mR_{\mY}\left(\W\right) + \act{\W , \nabla g\left(\X_{t} + \Y_{t}\right)} \right\}$,
		\STATE $\V_{t} = \argmin\limits_{\V \in \vecspace} \left\{ \phi_{t}\left(\V\right) := \mR_{\mX}\left(\V\right) + \act{\V , \nabla g\left(\X_{t} + \Y_{t}\right)} + \frac{\eta_{t}\beta}{2}\norm{\V + \W_{t} - \left(\X_{t} + \Y_{t}\right)}^{2} \right\}$, \\
		\COMMENT{in fact it suffices that $\phi_{t}\left(\V_{t}\right) \leq \phi_{t}\left(\X^{\ast}\right)$ for some optimal solution $\left(\X^{\ast} , \Y^{\ast}\right)$}
		\STATE $\left(\X_{t  + 1} , \Y_{t + 1}\right) = \left(1 - \eta_{t}\right)\left(\X_{t} , \Y_{t}\right) + \eta_{t}\left(\V_{t} , \W_{t}\right)$,
		\ENDFOR
	\end{algorithmic}
\end{algorithm}

Algorithm \ref{alg:GCG} is based on three well-known corner stones in continuous optimization: alternating minimization, conditional gradient, and proximal gradient. Since Problem \eqref{GeneralModel} involves two variables $\X$ and $\Y$, we update each one of them separately and differently in an alternating fashion. Indeed, the $\Y$ variable is first updated using a conditional gradient step (see step (4)) and then the alternating idea comes into a play and we use the updated information in order to update the $\X$ variable using a proximal gradient step (see step (5))\footnote{We note that a practical implementation of Algorithm \ref{alg:GCG} for a specific problem, such as Problem \eqref{eq:robustPCA}, may require to account for approximation errors in the computation of $\W_{t}$ or $\V_{t}$, since exact computation is not always practically feasible. Such considerations which can be easily incorporated both into Algorithm \ref{alg:GCG} and our corresponding analyses (see examples in \cite{Jaggi13b,Garber16a,allen2017linear}), are beyond the scope of this current paper, and for the simplicity and clarity of presentation, we assume all such computations are precise.}.

\subsection{Outline of the main results} \label{sec:results}
Let us denote by $f^{\ast}$ the optimal value of the optimization Problem \eqref{GeneralModel}. In the sequel we prove the following three theorems on the performance of Algorithm \ref{alg:GCG}. For clarity, below we present a concise and simplified version of the results. In section \ref{sec:analysis}, in which we provide complete proofs for these theorems, we also restate them with complete detail. In all three theorems we assume that Assumption \ref{Assumption1} holds true, and we bound the convergence rate of the sequence $\left\{ \left(\X_{t} , \Y_{t} \right)\right\}_{t \geq 1}$ produced by Algorithm \ref{alg:GCG} with a suitable choice of step-sizes $\left\{ \eta_{t} \right\}_{t \geq 1}$.
\begin{theorem} \label{thm:minDiam}
	Assume that $\mR_{\mY} := \iota_{\mY}$ where $\mY$ is a nonempty, closed and convex subset of $\vecspace$. There exists a choice of step-sizes such that Algorithm \ref{alg:GCG} converges with rate $O\left(\beta D_{\mY}^{2}/t\right)$.
\end{theorem}
\begin{remark}
	Note that since $\X$ and $\Y$ are in principle interchangeable, Theorem \ref{thm:minDiam} implies a rate of $O(\beta\min\{ D_{\mX}^{2} , D_{\mY}^{2} \}/t)$. This improves over the rate of $O(\beta\max\{ D_{\mX}^{2} , D_{\mY}^{2} \}/t)$ achieved by standard analyses of projected/proximal gradient methods and the conditional gradient method.
\end{remark}
\begin{theorem} \label{T:StrongSet}
	Assume $\mR_{\mX} := \iota_{\mX}$ where $\mX$ is a nonempty, closed and convex subset of $\vecspace$ and $\mR_{\mY} := \iota_{\mY}$, where $\mY$ is a strongly convex and closed subset of $\vecspace$. There exists a choice of step-sizes such that Algorithm \ref{alg:GCG} converges with rate $O(1/t^{2})$. Moreover, if there exists $G > 0$ such that $\min_{ \X \in \mX , \Y \in \mY} \norm{\nabla g\left(\X + \Y\right)}_{\ast} \geq G$, then using a fixed step-size, Algorithm \ref{alg:GCG} converges with rate  $O(\exp(-\Theta(t)))$.
\end{theorem}
\begin{remark} 
	While a rate of $O(1/t^{2})$ for the conditional gradient method over strongly convex sets was recently showed to hold in \cite{GH15}, it should be noted that it does not apply in the case of Theorem \ref{T:StrongSet}, since only the set $\mY$ is assumed to be strongly convex. In particular, both the set of sums $\mX + \mY \subset \vecspace$ and the product set $\mX \times \mY \subset \vecspace \times \vecspace$ need not be strongly convex.
\end{remark}
\begin{theorem} \label{T:StrongReg}
	Assume that $\mR_{\mY}(\cdot)$ is strongly convex. Then, there exists a fixed step-size such that Algorithm \ref{alg:GCG} converges with rate $O(\exp(-\Theta(t)))$.
\end{theorem}

\subsection{Putting our results in the context of Robust PCA problems} \label{sec:robustPCAexp}
As discussed in the Introduction, this work is mostly motivated by low-rank matrix optimization problems such as Robust PCA (see Problem \eqref{eq:robustPCA}). Thus, towards better understanding of our results for this setting, we now briefly detail the applications to Problem \eqref{eq:robustPCA}. As often standard in such problems, we assume that there exists an optimal solution $\left(\X^{\ast} , \Y^{\ast}\right)$ such that the signal matrix $\X^{\ast}$ is of rank at most $r^{\ast}$, where $r^{\ast} << \min\{ m , n \}$\footnote{Our results could be easily extended to the case in which $\left(\X^{\ast} , \Y^{\ast}\right)$ is nearly of rank $r^{\ast}$, i.e., of distance much smaller than the required approximation accuracy $\varepsilon$ to a rank $r^{\ast}$ matrix, however for the sake of clarity we simply assume that $\left(\X^{\ast} , \Y^{\ast}\right)$ is of low-rank.}.  

\subsubsection{Using low-rank SVD computations}\label{sec:lowranksvd}
Note that the computation of $\V_{t}$ in Algorithm \ref{alg:GCG}, which is used to update the estimate $\X_{t + 1}$, simply requires that $\V_{t}$ satisfies $\norm{\V_{t}}_{\nucnorm} \leq \tau$ and
\begin{equation} \label{eq:robPCAexp:1}
	\norm{\V_{t} - \left(\X_{t} + \Y_{t} - \W_{t} - \frac{1}{\eta_{t}}\nabla_{t}\right)}_{F}^{2} \leq \norm{\X^{\ast} - \left(\X_{t} + \Y_{t} - \W_{t} - \frac{1}{\eta_{t}}\nabla_{t}\right)}_{F}^{2},
\end{equation}
where we use the short notation $\nabla_{t} := \nabla g\left(\X_{t} + \Y_{t}\right)$. Since $\X^{\ast}$ is assumed to have rank at most $r^{\ast}$, it follows that
\begin{equation} \label{eq:robPCAexp:2}
	\textrm{RHS of \eqref{eq:robPCAexp:1} } \geq \min_{\X \in C} \norm{\X- \left(\X_{t} + \Y_{t} - \W_{t} - \frac{1}{\eta_{t}}\nabla_{t}\right)}_{F}^{2},
\end{equation}
where $C := \left\{ \X : \, \norm{\X}_{\nucnorm} \leq \tau , \, \rank(\X) \leq r^{\ast} \right\}$. The solution to the minimization problem on the RHS of \eqref{eq:robPCAexp:2} is given simply by computing the rank-$r^{\ast}$ singular value decomposition of the matrix $\A_{t} = \left(\X_{t} + \Y_{t} - \W_{t} -\left(1/\eta_{t}\right)\nabla_{t}\right)$, and projecting the resulted vector of singular values onto the $\ell_{1}$-norm ball of radius $\tau$ (which can be done in $O(r^{\ast}\log(r^{\ast}))$ time). Thus, indeed the time to compute the update for $\X_{t + 1}$ on each iteration of Algorithm \ref{alg:GCG} is dominated by a single rank-$r^{\ast}$ SVD computation. This observation holds for all the following discussions in this section as well. This low-rank SVD approach was already suggested in the recent work \cite{allen2017linear}, that studied smooth and strongly convex minimization over the nuclear-norm ball (which differs from our setting). 

\subsubsection{Improved complexities for low/high SNR regimes} \label{sec:SNR} 
In case that $\left(\X^{\ast} , \Y^{\ast}\right)$ is an (say, unique) optimal solution to Problem \eqref{eq:robustPCA}, which satisfies $\norm{\Y^{\ast}}_{F}^{2} << \norm{\X^{\ast}}_{F}^{2}$, i.e., a high signal-to-noise ratio regime, we expect that $D_{\mX} >> D_{\mY}$, where $D_{\mX}$ and $D_{\mY}$ are the Euclidean diameters of the nuclear norm ball and the $\ell_{1}$-norm ball, respectively. In this case, the result of Theroem \ref{thm:minDiam} is appealing since the convergence rate depends only on $D_{\mY}^{2}$ and not on $D_{\mX}^{2} + D_{\mY}^{2}$ as standard algorithms/analyses. In the opposite case, i.e., $D_{\mX} << D_{\mY}$, which naturally corresponds to a low signal-to-noise ratio regime, since $\X$ and $\Y$ are interchangeable in our setting, we can reverse their roles in the optimization and get via Theorem \ref{thm:minDiam} dependency only on $D_{\mX}$. Moreover, now the nuclear-norm constrained variable (assuming the role of $\Y$ in Algorithm \ref{alg:GCG}) is only updated via a conditional gradient update, i.e., requires only a rank-one SVD computation on each iteration. In particular, statistical recovery results such as the seminal work \cite{Candes11}, show that under suitable assumptions on the data, exact recovery is possible in both of these cases, even for instance, when $\norm{\Y^{\ast}}_{F}^{2}/\norm{\X^{\ast}}_{F}^{2} = \textrm{poly}(n)$.

\subsubsection{Replacing the $\ell_1$ constraint with a $\ell_{1+\delta}$ constraint} 
The $\ell_{1}$-norm is traditionally used in Robust PCA to constrain/regularize the sparse noisy component. The standard geometric intuition is that since the boundary of the $\ell_{1}$-norm ball becomes sharp near the axes, this choice promotes sparse solutions. This property also holds for an $\ell_{p}$-norm ball where $p$ is sufficiently close to 1. Thus, it might be reasonable to replace the $\ell_{1}$-norm constraint on $\Y$ with a $\ell_{1 + \delta}$-norm constraint for some small constant $\delta$, which results in a strongly convex feasible set for the variable $\Y$ (see \cite{GH15}). Using Theorem \ref{T:StrongSet}, we will obtain an improved convergence rate of $O(1/t^{2})$ instead of $O(1/t)$, practically without increasing the computational complexity per iteration (since $\Y$ is updated via a conditional gradient update and linear optimization over a $\ell_{p}$-norm ball can be carried-out in linear time \cite{GH15}). 
\medskip

In order to demonstrate the plausibility of using the $\ell_{1 + \delta}$-norm instead of $\ell_{1}$-norm, in Table \ref{table:Lp} we present results on synthetic data (similar to those done in \cite{Candes11}), which show that already for a moderate value of $\delta = 0.2$ we obtain quite satisfactory recovery results.

\begin{table*}\renewcommand{\arraystretch}{1.3}
	\begin{center}{\footnotesize
  		\begin{tabular}{| l | c | c | c | c | c |} \hline
     		& $\delta = 0.05$ & $\delta = 0.1$ & $\delta = 0.2$ & $\delta = 0.3$ & $\delta = 0.4$ \\ \hline
    			$\norm{\X - \X^{\ast}}_{F}^{2}/\norm{\X^{\ast}}_{F}^{2}$ & $9.2 \times 10^{-5}$ & $6.0 \times 10^{-4}$ & $4.3 \times 10^{-2}$ & $0.25$ & $0.61$ \\ \hline
    			$\norm{\Y - \Y^{\ast}}_{F}^{2}/\norm{\Y^{\ast}}_{F}^{2}$ & $1.5 \times 10^{-6}$ & $6.4 \times 10^{-6}$ & $4.2 \times 10^{-4}$ & $2.5 \times 10^{-3}$ & $6.1 \times 10^{-3}$ \\ \hline
  		\end{tabular}
  		\caption{Empirical results for solving Problem \eqref{eq:robustPCA} with $\ell_{1 + \delta}$-norm constraint on $\Y$ instead of $\ell_{1}$-norm. The input matrix is $\M = \X^{\ast} + \Y^{\ast}$, where $\X^{\ast} = \U\V^{\top}$ for $\U , \V \in \reals^{n \times r}$ with entries sampled i.i.d. from $\mathcal{N}(0 , 1/n)$ with $n = 100$ and $r = 10$. Every entry in $\Y^{\ast}$ is set i.i.d. to $0$ w.p. $0.9$ and to either $+1$ or $-1$ w.p. $0.05$. For simplicity we set exact bounds $\tau = \norm{\X^{\ast}}_{\nucnorm}$ and $s = \norm{\Y^{\ast}}_{1 + \delta}$. First row gives the relative recovery error of the low-rank signal and the second row gives the relative recovery error of the sparse noise component, averaged over 10 i.i.d. experiments.}
		\label{table:Lp}}
		\end{center}
		\vskip -0.2in
\end{table*}\renewcommand{\arraystretch}{1}

\subsubsection{Replacing the $\ell_{1}$-norm regularizer with an elastic net regularizer} 
In certain cases it may be beneficial to replace the $\ell_{1}$-norm constraint (or regularizer) of the variable $\Y$ in Problem \eqref{eq:robustPCA} with an elastic net regularizer, i.e., to take $\mR_{\mY}\left(\Y\right) = \lambda_{1}\norm{\Y}_{1} + \lambda_{2}\norm{\Y}_{F}^{2}$, for some $\lambda_{1} , \lambda_{2} > 0$. The elastic net is a popular alternative to the standard $\ell_{1}$-norm regularizer for problems such as LASSO (see, for instance \cite{zou2005regularization}). As opposed to the $\ell_{1}$-norm regularizer, the elastic net is strongly convex (though not differentiable). Thus, with such a choice for $\mR_{\mY}\left(\cdot\right)$, by invoking Theorem \ref{T:StrongReg}, Algorithm \ref{alg:GCG} guarantees a linear convergence rate. We note that when using the elastic net regularizer, the computation of $\W_{t}$ on each iteration of Algorithm \ref{alg:GCG} requires to solve the optimization problem:
\begin{equation*}
	\argmin_{\W \in \vecspace} \left\{ \act{\W  , \nabla_{t}} + \lambda_{1}\norm{\W}_{1} + \lambda_{2}\norm{\W}_{F}^{2} \right\} = \argmin_{\W \in \vecspace} \left\{ \lambda_{1}\norm{\W}_{1} + \lambda_{2}\norm{\W + \frac{1}{2\lambda_{2}}\nabla_{t}}_{F}^{2} \right\},
\end{equation*}
where we again use the short notation $\nabla_{t} = \nabla g\left(\X_{t} + \Y_{t}\right)$. In the optimization problem above, the RHS admits a well-known closed-form solution given by the \textit{shrinkage/soft-thresholding operator}, which can be computed in linear time (i.e., $O(mn)$ time), see for instance \cite{FISTA}.

\section{Rate of Convergence Analysis}\label{sec:analysis}
In this section we provide the proofs for Theorems \ref{thm:minDiam}, \ref{T:StrongSet}, and \ref{T:StrongReg}. Throughout this section and for the simplicity of the yet to come developments we denote, for all $t \geq 1$, $\Z_{t} := \X_{t} + \Y_{t}$, $\U_{t} := \V_{t} + \W_{t}$, and $\Q_{t} := \left(\X_{t} , \Y_{t}\right)$. Note that, using these notations we obviously have that $\Z_{t + 1} = \left(1  - \eta_{t}\right)\Z_{t} + \eta_{t}\U_{t}$. Similarly, for an optimal solution $\Q^{\ast} :=\left(\X^{\ast} , \Y^{\ast}\right)$ of Problem \eqref{GeneralModel} we denote $\Z^{\ast} := \X^{\ast} + \Y^{\ast}$.
\medskip

We will need the following technical result which forms the basis for the proofs of all stated theorems. 
\begin{proposition} \label{P:Tech}
	Let $\left\{ \left(\X_{t} , \Y_{t}\right) \right\}_{t \geq 1}$ be a sequence generated by Algorithm \ref{alg:GCG}. Then, for all $t \geq 1$, we have that
\begin{align}
	f\left(\Q_{t + 1}\right) & \leq \left(1 - \eta_{t}\right)f\left(\Q_{t}\right) + \eta_{t}\left(g\left(\Z_{t}\right) + \mR_{\mX}\left(\X^{\ast}\right) + \mR_{\mY}\left(\W_{t}\right)\right) \nonumber \\
	& + \eta_{t}\act{\X^{\ast} + \W_{t} - \Z_{t} , \nabla{}g\left(\Z_{t}\right)} + \eta_{t}^{2}\beta\left(\norm{\Z_{t} - \Z^{\ast}}^{2} + \norm{\W_{t} - \Y^{\ast}}^{2}\right). \label{P:Tech:1}
\end{align}
\end{proposition}
\begin{proof}
Fix $t \geq 1$. Observe that by the update rule of Algorithm \ref{alg:GCG} (see step 6 of the algorithm), it holds that 
\begin{equation*}
	\X_{t + 1} + \Y_{t + 1} = \left(1 - \eta_{t}\right)\left(\X_{t} + \Y_{t}\right) + \eta_{t}\left(\V_{t} + \W_{t}\right).
\end{equation*}
Thus, since $g\left(\cdot\right)$ is $\beta$-smooth, it follows that
\begin{align*}
	g\left(\X_{t + 1} + \Y_{t + 1}\right) & \leq g\left(\X_{t} + \Y_{t}\right) + \eta_{t}\act{\left(\V_{t} + \W_{t}\right) - \left(\X_{t} + \Y_{t}\right) , \nabla g\left(\X_{t} + \Y_{t}\right)} \\
	& + \frac{\eta_{t}^{2}\beta}{2}\norm{\left(\X_{t} + \Y_{t}\right) - \left(\V_{t} + \W_t\right)}^{2} \\
	& = g\left(\Z_t\right) + \eta_{t}\act{\V_{t} + \W_{t} - \Z_{t} , \nabla g\left(\Z_{t}\right)} + \frac{\eta_{t}^{2}\beta}{2}\norm{\Z_{t} - \V_{t} - \W_{t}}^{2}. 
\end{align*}
Using the above inequality we can write,
\begin{align} \label{eq:prop:1}
	f\left(\X_{t + 1} , \Y_{t + 1}\right) & = g\left(\X_{t + 1} + \Y_{t + 1}\right) + \mR_{\mX}\left(\X_{t   + 1}\right) + \mR_{\mY}\left(\Y_{t + 1}\right) \nonumber\\
	& \leq g\left(\Z_t\right) + \mR_{\mX}\left(\X_{t + 1}\right) + \mR_{\mY}\left(\Y_{t + 1}\right) + \eta_{t}\act{\V_{t} + \W_{t} - \Z_{t} , \nabla g\left(\Z_t\right)} \nonumber \\
	& + \frac{\eta_{t}^{2}\beta}{2}\norm{\Z_{t} - \V_{t} - \W_{t}}^{2} \nonumber\\
	& \underset{(a)}{\leq} \left(1 - \eta_{t}\right)\left(g\left(\Z_{t}\right) + \mR_{\mX}\left(\X_{t}\right) + \mR_{\mY}\left(\Y_{t}\right)\right) \nonumber \\
	& + \eta_{t}\left(g\left(\Z_{t}\right) + \mR_{\mX}\left(\V_{t}\right) + \mR_{\mY}\left(\W_{t}\right)\right) \nonumber \\
	& + \eta_{t}\act{\V_{t} + \W_{t} - \Z_{t} , \nabla g\left(\Z_{t}\right)} + \frac{\eta_{t}^{2}\beta}{2}\norm{\X_{t} + \Y_{t} - \V_{t} - \W_{t}}^{2} \nonumber \\
	& = \left(1 - \eta_{t}\right)f\left(\X_{t} , \Y_{t}\right) + \eta_{t}\left(g\left(\Z_{t}\right) + \mR_{\mX}\left(\V_{t}\right) + \mR_{\mY}\left(\W_{t}\right)\right) \nonumber \\
	& + \eta_{t}\act{\V_{t} + \W_{t} - \Z_{t} , \nabla g\left(\Z_{t}\right)} + \frac{\eta_{t}^{2}\beta}{2}\norm{\Z_{t} - \V_{t} - \W_{t}}^{2} \nonumber \\
	& \underset{(b)}{\leq} \left(1 - \eta_{t}\right)f\left(\X_{t} , \Y_{t}\right) + \eta_{t}\left(g\left(\Z_{t}\right) + \mR_{\mX}\left(\X^{\ast}\right) + \mR_{\mY}\left(\W_{t}\right)\right) \nonumber \\
	& + \eta_{t}\act{\X^{\ast} + \W_{t} - \Z_{t} , \nabla g\left(\Z_{t}\right)} + \frac{\eta_{t}^{2}\beta}{2}\norm{\Z_{t} - \X^{\ast} - \W_{t}}^{2},
\end{align}
where (a) follows from the convexity of $\mR_{\mX}\left(\cdot\right)$ and $\mR_{\mY}\left(\cdot\right)$, while (b) follows from the choice of $\V_{t}$. Using the inequality $\left(a + b\right)^{2} \leq 2a^{2} + 2b^{2}$ which holds true for all $a , b \in \reals$, we obtain
\begin{align} \label{eq:prop:2}
	\norm{\Z_{t} - \X^{\ast} + \W_{t}}^{2} & = \norm{\Z_{t} - \X^{\ast} - \Y^{\ast} + \left(\Y^{\ast} - \W_{t}\right)}^{2} \nonumber \\
	& \leq 2\norm{\Z_{t} - \Z^{\ast}}^{2} + 2\norm{\W_{t} - \Y^{\ast}}^{2},
\end{align}
where the last equality follows from the definitions of $\Z_{t}$ and $\Z^{\ast}$.

Finally, plugging the RHS of Eq. \eqref{eq:prop:2} into the RHS of Eq. \eqref{eq:prop:1} completes the proof of the proposition.
\end{proof}
We now prove Theorem \ref{thm:minDiam}. For convenience, we first state the theorem in full details.
\begin{theorem}\label{thm:minDiam:detail}
Assume that $\mR_{\mY} := \iota_{\mY}$ where $\mY$ is a nonempty, closed and convex subset of $\vecspace$. Let $\left\{ \left(\X_{t} , \Y_{t}\right) \right\}_{t \geq 1}$ be a sequence generated by Algorithm \ref{alg:GCG} with the following step-sizes:
\begin{equation}\label{eq:mindiam:step}
	\eta_{t} = 
		\begin{cases}
        		\frac{\alpha}{2\beta}, & \mbox{if }t \leq t_{0}, \\
        		\frac{2}{t - t_{0} + \frac{4\beta}{\alpha}}, & \mbox{if }t > t_{0},
       \end{cases}
\end{equation} 
where $t_{0} := \max \left\{ 0 , \left\lceil{ 2\beta/\left(\alpha\right)\ln\left(2C/\left(\alpha D_{\mY}^{2}\right)\right)}\right\rceil \right\}$, for $C$ satisfying $C \geq f\left(\X_{1} , \Y_{1}\right) - f^{\ast}$. Then, for all $t \geq t_{0} + 1$ it holds that
\begin{equation*}
	f\left(\X_{t} , \Y_{t}\right) - f\left(\X^{\ast} , \Y^{\ast}\right) \leq \frac{4\beta{}D_{\mY}^{2}}{t - t_{0} - 1 + \frac{4\beta}{\alpha}}.
\end{equation*}
\end{theorem}
\begin{proof}
From the choice of $\W_{t}$ we have that
\begin{equation} \label{eq:thmDiam:1}
	\act{\W_{t} , \nabla g\left(\Z_{t}\right)} + \mR_{\mY}\left(\W_{t}\right) \leq \act{\Y^{\ast} , \nabla g\left(\Z_{t}\right)} + \mR_{\mY}\left(\Y^{\ast}\right).
\end{equation}
Now, using this in Proposition \ref{P:Tech}, we get for all $t \geq 1$, that
\begin{align}\label{eq:thmDiam:2}
	f\left(\Q_{t + 1}\right) & \leq \left(1 - \eta_{t}\right)f\left(\Q_{t}\right) + \eta_{t}\left(g\left(\Z_{t}\right) + \mR_{\mY}\left(\Y^{\ast}\right) + \mR_{\mX}\left(\X^{\ast}\right)\right) \nonumber \\
	& + \eta_{t}\act{\X^{\ast} + \Y^{\ast} - \Z_{t} , \nabla g\left(\Z_{t}\right)} + \eta_{t}^{2}\beta\left(\norm{\Z_{t} - \Z^{\ast}}^{2} + \norm{\W_{t} - \Y^{\ast}}^{2}\right) \nonumber \\
	& \leq \left(1 - \eta_{t}\right)f\left(\Q_{t}\right) + \eta_{t}\left(g\left(\Z_{t}\right) + \mR_{\mY}\left(\Y^{\ast}\right) + \mR_{\mX}\left(\X^{\ast}\right)\right) \nonumber \\
	& + \eta_{t}\act{\X^{\ast} + \Y^{\ast} - \Z_{t} , \nabla g\left(\Z_{t}\right)} + \eta_{t}^{2}\beta\left(\norm{\Z_{t} - \Z^{\ast}}^{2} + D_{\mY}^{2}\right),  
\end{align}
where the last inequality follows from the fact that $\W_{t} , \Y^{\ast} \in \mY$. On the other hand, from the strong convexity of $g\left(\cdot\right)$ we obtain that
\begin{equation*}
	g\left(\Z_{t}\right) + \act{\X^{\ast} + \Y^{\ast} - \Z_{t} , \nabla g\left(\Z_{t}\right)} \leq g\left(\Z^{\ast}\right) - \frac{\alpha}{2}\norm{\Z_{t} - \Z^{\ast}}^{2}.
\end{equation*} 
Therefore, by combining these two inequalities we derive that
\begin{equation*}
	f\left(\Q_{t + 1}\right) \leq \left(1 - \eta_{t}\right)f\left(\Q_{t}\right) + \eta_{t}f\left(\Q^{\ast}\right) - \eta_{t}\left(\frac{\alpha}{2} - \eta_{t}\beta\right)\norm{\Z_{t} - \Z^{\ast}}^{2} + \eta_{t}^{2}\beta D_{\mY}^{2}.
\end{equation*}
Subtracting $f\left(\Q^*\right)$ from both sides of the inequality and by denoting $h_{t} := f\left(\Q_{t}\right) - f\left(\Q^*\right)$, we obtain that $h_{t + 1} \leq \left(1 - \eta_{t}\right)h_{t} + \eta_{t}^{2}\beta D_{\mY}^{2}$ holds true for all $0 < \eta_{t} \leq \alpha/\left(2\beta\right)$. The result now follows from simple induction arguments and the choice of step-sizes detailed in the theorem (for details see Lemma \ref{L:Reco} in the appendix below). 
\end{proof}
Before proving Theorem \ref{T:StrongSet} we would like to comment about the constant $C$, which was used in the result above and appears in the step-size.
\begin{remark}
	The constant $C$, even though appears in the step-size of the algorithm, can be easily bounded from above as we describe now. Suppose, we are setting the points $\W_{1}$ and $\V_{1}$ to be used in our algorithm as follows:
	\begin{equation*}
		\W_{1} = \argmin\limits_{\W \in \vecspace} \left\{ \mR_{\mY}\left(\W\right) + \act{\W , \nabla g\left(\X_{1} + \Y_{1}\right)} \right\},
	\end{equation*}
	and
	\begin{equation*}
		\V_{1} = \argmin\limits_{\V \in \vecspace} \left\{ \mR_{\mX}\left(\V\right) + \act{\V , \nabla g\left(\X_{1} + \Y_{1}\right)} \right\}.
	\end{equation*}
	For these choices we obviously have (using optimality conditions) that
	\begin{align*}
		\mR_{\mY}\left(\W_{1}\right) + \mR_{\mX}\left(\V_{1}\right) + \act{\W_{1} + \V_{1} , \nabla g\left(\X_{1} + \Y_{1}\right)} & \leq \mR_{\mY}\left(\Y^{\ast}\right) + \mR_{\mX}\left(\X^{\ast}\right) \\
		& + \act{\Y^{\ast} + \X^{\ast} , \nabla g\left(\X_{1} + \Y_{1}\right)}.
	\end{align*}
	Hence, using the gradient inequality on the function $g$, yields that
	\begin{align*}
		f\left(\X_{1} , \Y_{1}\right) - f^{\ast} & = g\left(\X_{1} , \Y_{1}\right) - g\left(\X^{\ast} , \Y^{\ast}\right) + \mR_{\mY}\left(\Y_{1}\right) + \mR_{\mX}\left(\X_{1}\right) - \mR_{\mY}\left(\Y^{\ast}\right) - \mR_{\mX}\left(\X^{\ast}\right) \\
		& \leq \act{\left(\X_{1} + \Y_{1}\right) - \left(\Y^{\ast} + \X^{\ast}\right) , \nabla g\left(\X_{1} + \Y_{1}\right)} + \mR_{\mY}\left(\Y_{1}\right) + \mR_{\mX}\left(\X_{1}\right) \\
		& - \mR_{\mY}\left(\W_{1}\right) - \mR_{\mX}\left(\V_{1}\right) + \act{\left(\Y^{\ast} + \X^{\ast}\right) - \left(\W_{1} + \V_{1}\right) , \nabla g\left(\X_{1} + \Y_{1}\right)} \\
		& = \act{\left(\X_{1} + \Y_{1}\right) - \left(\W_{1} + \V_{1}\right) , \nabla g\left(\X_{1} + \Y_{1}\right)} + \mR_{\mY}\left(\Y_{1}\right) + \mR_{\mX}\left(\X_{1}\right) \\
		& - \mR_{\mY}\left(\W_{1}\right) - \mR_{\mX}\left(\V_{1}\right).
	\end{align*}
	The obtained bound does not depend on the optimal solution and therefore can be computed explicitly. 
	
	It should be noted that in the case of Robust PCA (e.g., Problem \eqref{eq:robustPCA}), we have that $\mR_{\mX}\left(\X\right) = \iota_{\norm{\cdot}_{\nucnorm} \leq \tau}\left(\X\right)$ and $\mR_{\mY}\left(\Y\right) = \iota_{\norm{\cdot}_{1} \leq s}\left(\Y\right)$. In this case, computing the matrices $\W_{1}$ and $\V_{1}$ is computationally very efficient, since it requires to compute a single leading singular vectors pair, and solving a single linear minimization problem over an $\ell_1$-ball, respectively.
\end{remark}
Now, we turn to prove Theorem \ref{T:StrongSet}. Again, we first state the theorem in full details.
\begin{theorem}
Assume that $\mR_{\mX} := \iota_{\mX}$ where $\mX$ is a nonempty, closed and convex subset of $\vecspace$ and $\mR_{\mY} := \iota_{\mY}$, where $\mY$ is an $\gamma$-strongly convex and closed subset of $\vecspace$. Let $\left\{ \left(\X_{t} , \Y_{t}\right) \right\}_{t \geq 1}$ be a sequence produced by Algorithm \ref{alg:GCG} using the step-size $\eta_{t} = \left(t - 1 + 6\beta/\alpha\right)^{-1}$ for all $t \geq 1$. Then, for all $t \geq 1$ it holds that
\begin{equation*}
	f\left(\X_{t} , \Y_{t}\right) - f^{\ast} \leq \frac{9\max\left\{ \frac{128\beta^{2}}{\alpha\gamma^{2}} , \frac{4\beta^{2}}{\alpha^{2}}\left(f\left(\X_{1} , \Y_{1}\right) - f^{\ast}\right) \right\}}{\left(t - 1 + 6\frac{\beta}{\alpha}\right)^{2}}.
\end{equation*}
Moreover, if there exists $G > 0$ such that $\min_{\X \in \mX , \Y \in \mY} \norm{\nabla g\left(\X + \Y\right)}_{\ast} \geq G$, then using a fixed step-size $\eta_{t} = \min\{ \alpha/\left(2\beta\right) ,\gamma G/\left(8\beta\right) \}$ for all $t \geq 1$, guarantees that 
\begin{equation*}
	f\left(\X_{t} , \Y_{t}\right) - f^{\ast} \leq \left(f\left(\X_{1} , \Y_{1}\right) - f^{\ast}\right)\cdot \exp\left(-\min\left\{ \frac{\alpha}{2\beta} , \frac{\gamma G}{8\beta}\right\}\left(t - 1\right)\right).
\end{equation*}
\end{theorem}
\begin{proof}
Fix some iteration $t \geq 1$ and define the point $\tilde{\W}_{t} := \frac{1}{2}\left(\W_{t} + \Y^{\ast}\right) - \left(\gamma/8\right)\norm{\W_{t} - \Y^{\ast}}^{2}\P_{t}$ where $\P_{t} \in \argmin_{\P \in \vecspace, \norm{\P} \leq 1} \act{\P , \nabla g\left(\Z_{t}\right)}$. Note that since $\mY$ is an $\gamma$-strongly convex set, it follows from Definition \ref{def:strongconvexset} that $\tilde{\W}_{t} \in \mY$. Moreover, from the optimal choice of $\W_{t}$ we have that $\act{\W_{t} , \nabla g\left(\Z_{t}\right)} \leq \act{\tilde{\W}_{t} , \nabla g\left(\Z_{t}\right)}$. Thus, we have that
\begin{align} \label{eq:strongset:1}
	\act{\X^{\ast} + \W_{t} - \Z_{t} , \nabla g\left(\Z_{t}\right)} & \leq \act{\X^{\ast} + \tilde{\W}_{t} - \Z_{t} , \nabla g\left(\Z_{t}\right)} \nonumber \\
	& = act{\X^{\ast} - \X_{t} + \frac{\W_{t} + \Y^{\ast}}{2} - \frac{\gamma}{8}\norm{\W_{t} - \Y^{\ast}}^{2}\P_{t} - \Y_{t} , \nabla g\left(\Z_{t}\right)} \nonumber \\
	& \underset{(a)}{\leq} \act{\X^{\ast} - \X_{t} + \frac{\Y^{\ast} + \Y^{\ast}}{2} - \frac{\gamma}{8}\norm{\W_{t} - \Y^{\ast}}^{2}\P_{t} - \Y_{t} , \nabla g\left(\Z_{t}\right)} \nonumber\\
	& \underset{(b)}{=} \act{\Z^{\ast} - \Z_{t} , \nabla g\left(\Z_{t}\right)} - \frac{\gamma}{8}\norm{\W_{t} - \Y^{\ast}}^{2}\cdot\norm{\nabla g\left(\Z_{t}\right)}_{\ast},
\end{align}
where (a) follows from the fact that $\act{\W_{t} - \Y^{\ast}  , \nabla g\left(\Z_{t}\right)} \leq 0$ , and (b) follows from the definition of $\P_{t}$ and Holder's inequality.

Plugging Eq. \eqref{eq:strongset:1} into Eq. \eqref{P:Tech:1}, and recalling that $\W_{t} \in \mY$ (hence, $\mR_{\mY}\left(\W_{t}\right) = \mR_{\mY}\left(\Y^{\ast}\right) = \mR_{\mX}\left(\X^{\ast}\right) = 0$), we have that
\begin{align}
	f\left(\X_{t + 1} , \Y_{t + 1}\right) & \leq \left(1 - \eta_{t}\right)f\left(\X_{t} , \Y_{t}\right) + \eta_{t}g\left(\Z_{t}\right) + \eta_{t}\act{\X^{\ast} + \W_{t} - \Z_{t} , \nabla g\left(\Z_{t}\right)}\nonumber \\
	& + \eta_{t}^{2}\beta\left(\norm{\Z_{t} - \Z^{\ast}}^{2} + \norm{\W_{t} - \Y^{\ast}}^{2}\right) \nonumber \\
	& \leq \left(1 - \eta_{t}\right)f\left(\X_{t} , \Y_{t}\right) + \eta_{t}g\left(\Z_{t}\right) + \eta_{t}\act{\Z^{\ast} - \Z_{t} , \nabla g\left(\Z_{t}\right)} \nonumber \\
	& + \eta_{t}^{2}\beta\norm{\Z_{t} - \Z^{\ast}}^{2} - \eta_{t}\norm{\W_{t} - \Y^{\ast}}^{2}\left(\frac{\gamma}{8}\norm{\nabla g\left(\Z_{t}\right)}_{\ast} - \eta_{t}\beta\right) \nonumber \\
	& \underset{(a)}{\leq} \left(1 - \eta_{t}\right)f\left(\X_{t} , \Y_{t}\right) + \eta_{t}g\left(\Z^{\ast}\right) - \eta_{t}\norm{\Z_{t} - \Z^{\ast}}^{2}\left(\frac{\alpha}{2} - \eta_{t}\beta\right) \nonumber \\
	& - \eta_{t}\norm{\W_{t} - \Y^{\ast}}^{2}\left(\frac{\gamma}{8}\norm{\nabla g\left(\Z_{t}\right)}_{\ast} - \eta_{t}\beta\right) \nonumber \\
	& = \left(1 - \eta_{t}\right)f\left(\X_{t} , \Y_{t}\right) + \eta_{t}f\left(\X^{\ast} , \Y^{\ast}\right) - \eta_{t}\norm{\Z_{t} - \Z^{\ast}}^{2}\left(\frac{\alpha}{2} - \eta_{t}\beta\right) \nonumber \\
	& - \eta_{t}\norm{\W_{t} - \Y^{\ast}}^{2}\left(\frac{\gamma}{8}\norm{\nabla g\left(\Z_{t}\right)}_{\ast} - \eta_{t}\beta\right), \label{T:StrongSet:2}
\end{align}
where (a) follows from the strong convexity of $g(\cdot)$ since we have that 
\begin{equation*}
	\act{\Z^{\ast} - \Z_{t} , \nabla g\left(\Z_{t}\right)} \leq g\left(\Z^{\ast}\right) - g\left(\Z_{t}\right) - \frac{\alpha}{2}\norm{\Z_{t} - \Z^{\ast}}^{2}.
\end{equation*}
Using again the strong convexity of $g\left(\cdot\right)$, we have from Eq. \eqref{ie:largegrad} that
\begin{align*}
	\norm{\nabla g\left(\Z_{t}\right)}_{\ast} & \geq \sqrt{\frac{\alpha}{2}} \cdot \sqrt{g\left(\Z_{t}\right) - g\left(\Z^{\ast}\right)} \underset{(a)}{=} \sqrt{\frac{\alpha}{2}} \cdot \sqrt{f\left(\X_{t} , \Y_{t}\right) - f\left(\X^{\ast} , \Y^{\ast}\right)} \\
	& = \sqrt{\frac{\alpha}{2}} \cdot \sqrt{h_{t}},
\end{align*}
where $h_{t} := f\left(\X_{t} , \Y_{t}\right) - f\left(\X^{\ast} , \Y^{\ast}\right)$, and (a) follows since $\mR_{\mX}\left(\X_{t}\right) = \mR_{\mY}\left(\Y_{t}\right) = \mR_{\mX}\left(\X^{\ast}\right) = \mR_{\mY}\left(\Y^{\ast}\right) = 0$. Therefore, by subtracting $f\left(\X^{\ast} , \Y^{\ast}\right)$ from both sides of \eqref{T:StrongSet:2}, we get that
\begin{align*}
	h_{t + 1} & \leq \left(1 - \eta_{t}\right)h_{t} - \eta_{t}\left(\frac{\alpha}{2} - \eta_{t}\beta\right)\norm{\Z_{t} - \Z^{\ast}}^{2} \\
	& - \eta_{t}\norm{\W_{t} - \Y^{\ast}}^{2}\left(\frac{\gamma}{8}\sqrt{\frac{\alpha}{2}} \cdot\sqrt{h_{t}} - \eta_{t}\beta\right).
\end{align*}
Thus, we obtain the recursion: $h_{t + 1} \leq \left(1 - \eta_{t}\right)h_{t}$ for all $\eta_{t} \leq \min\left\{ \gamma\sqrt{\alpha}\sqrt{h_{t}}/\left(8\sqrt{2}\beta\right) , \alpha/\left(2\beta\right) \right\}$.

In particular, setting $\eta_{t}$ as stated in the theorem yields the stated convergence rate via a simple induction argument, given Lemma \ref{lem:strongset:recurs} (see appendix for a proof).

In order to prove the second part of the theorem, i.e., a linear convergence in the case that the gradients are bounded from below, we observe that plugging the bound $G$ on the magnitude of the gradients into the RHS of Eq. \eqref{T:StrongSet:2}, directly gives
\begin{align*}
	f\left(\X_{t + 1} , \Y_{t + 1}\right) & \leq \left(1 - \eta_{t}\right)f\left(\X_{t} , \Y_{t}\right) + \eta_{t}f\left(\X^{\ast} , \Y^{\ast}\right) \nonumber \\
	& - \eta_{t}\norm{\Z_{t} - \Z^{\ast}}^{2}\left(\frac{\alpha}{2} - \eta_{t}\beta\right) - \eta_{t}\norm{\W_{t} - \Y^{\ast}}^{2}\left(\frac{\gamma G}{8} - \eta_{t}\beta\right).
\end{align*}
Thus, for any $\eta_{t} \leq \min\left\{ \alpha/\left(2\beta\right) , \gamma G/\left(8\beta\right) \right\}$, by subtracting $f\left(\X^{\ast} , \Y^{\ast}\right)$ from both sides, we obtain $h_{t + 1} \leq \left(1 -\eta_{t}\right)h_{t}$. In particular, setting $\eta_{t} = \min \left\{ \alpha/\left(2\beta\right) , \gamma G/\left(8\beta\right) \right\}$ for all $t \geq 1$ and using elementary manipulations, gives the linear rate stated in the theorem.
\end{proof}
Before stating in details Theorem \ref{T:StrongReg} and proving it, we would like to prove that the additional assumption made in this result, i.e., that $\mR_{\mX}\left(\cdot\right)$ or $\mR_{\mY}\left(\cdot\right)$ is $\delta$-strongly convex, actually guarantees that the whole objective function $f\left(\cdot , \cdot\right)$ is also strongly convex. The following result is valid when the function $g$ is strongly convex with respect to the Euclidean norm. 
\begin{proposition} \label{P:fStrongConv}
Assume that $\mR_{\mX}$ or $\mR_{\mY}$ is $\delta$-strongly convex with respect to the Euclidean norm. Then, the objective function $f\left(\cdot , \cdot\right)$ of Problem \eqref{GeneralModel}, is $\tau$-strongly convex with respect to the Euclidean norm, where $\tau = \left(\delta + 2\alpha - \sqrt{\delta^{2} + 4\alpha^{2}}\right)/2$.
\end{proposition}
\begin{proof}
Throughout the proof we let $\norm{\cdot}$ denote the Euclidean norm over $\vecspace$. Without the loss of generality we assume that $\mR_{\mX}$ is $\delta$-strongly convex. Let $\Q_{1} = \left(\X_{1} , \Y_{1}\right)$ and $\Q_{2} = \left(\X_{2} , \Y_{2}\right)$ be two points in $\vecspace \times \vecspace$ and $\lambda \in \left[0 , 1\right]$. Then, by the definition of strong convexity, we have that
\begin{equation*}
	\mR_{\mX}\left(\lambda\X_{1} + \left(1 - \lambda\right)\X_{2}\right) \leq \lambda\mR_{\mX}\left(\X_{1}\right) + \left(1 - \lambda\right)\mR_{\mX}\left(\X_{2}\right)  - \frac{\lambda\left(1 - \lambda\right)\delta}{2}\norm{\X_{1} - \X_{2}}^{2},
\end{equation*}
and
\begin{equation*}
	g\left(\P\right) \leq \lambda g\left(\X_{1} + \Y_{1}\right) + \left(1 - \lambda\right)g\left(\X_{2} + \Y_{2}\right)  - \frac{\lambda\left(1 - \lambda\right)\alpha}{2}\norm{\X_{1} + \Y_{1} - \X_{2} - \Y_{2}}^{2},
\end{equation*}
where $\P = \lambda\left(\X_{1} + \Y_{1}\right) + \left(1 - \lambda\right)\left(\X_{2} + \Y_{2}\right)$. On the other hand, for any $s > 0$, we have that
\begin{equation*}
	\norm{\X_{1} + \Y_{1} - \X_{2} - \Y_{2}}^{2} \geq \left(1 - s\right)\norm{\X_{1} - \X_{2}}^{2} + \frac{s - 1}{s}\norm{\Y_{1} - \Y_{2}}^{2},
\end{equation*}
where we have used the fact that $\left(a + b\right)^{2} \geq \left(1 - s\right)a^{2} + \left(s - 1\right)b^{2}/s$ for all $a , b \in \reals$, and that the norm is the Euclidean norm. Combining all these facts and using the fact that $\mR_{\mX}\left(\cdot\right)$ is convex yields that
\begin{align*}
	f\left(\lambda\Q_{1} + \left(1 - \lambda\right)\Q_{2}\right) & = g\left(\P\right) + \mR_{\mX}\left(\lambda\X_{1} + \left(1 - \lambda\right)\X_{2}\right) + \mR_{\mY}\left(\lambda\Y_{1} + \left(1 - \lambda\right)\Y_{2}\right) \\
	& \leq \lambda f\left(\Q_{1}\right) + \left(1 - \lambda\right)f\left(\Q_{2}\right) - \frac{\lambda\left(1 - \lambda\right)\delta}{2}\norm{\X_{1} - \X_{2}}^{2} \\
	& - \frac{\lambda\left(1 - \lambda\right)\alpha\left(1 - s\right)}{2}\norm{\X_{1} - \X_{2}}^{2} - \frac{\lambda\left(1 - \lambda\right)\alpha\left(s - 1\right)}{2s}\norm{\Y_{1} - \Y_{2}}^{2} \\
	& = \lambda f\left(\Q_{1}\right) + \left(1 - \lambda\right)f\left(\Q_{2}\right) - \frac{\lambda\left(1 - \lambda\right)}{2}\left(\delta + \alpha\left(1 - s\right)\right)\norm{\X_{1} - \X_{2}}^{2} \\
	& - \frac{\lambda\left(1 - \lambda\right)\alpha\left(s - 1\right)}{2s}\norm{\Y_{1} - \Y_{2}}^{2} \\
	& \leq \lambda f\left(\Q_{1}\right) + \left(1 - \lambda\right)f\left(\Q_{2}\right) - \frac{\lambda\left(1 - \lambda\right)\tau(s)}{2}\norm{\Q_{1} - \Q_{2}}^{2},
\end{align*}
where $\tau(s) = \min\left\{ \delta + \alpha\left(1 - s\right) , \alpha\left(s - 1\right)/s \right\}$ and the last inequality follows from the definitions of $\Q_{1}$ and $\Q_{2}$. It is easy to check that $\tau(s)$ gets its maximum with respect to $s$ when $s = \left(\delta + \sqrt{\delta^{2} + 4\alpha^{2}}\right)/\left(2\alpha\right)$. Therefore we get that $f\left(\cdot , \cdot\right)$ is strongly convex with parameter $\left(\delta + 2\alpha - \sqrt{\delta^{2} + 4\alpha^{2}}\right)/2$.
\end{proof}
Thanks to Proposition \ref{P:fStrongConv}, Problem \eqref{GeneralModel} becomes an unconstrained minimization of a strongly convex function. Therefore, we can expect to achieve a linear rate of convergence of Algorithm \ref{alg:GCG} as we prove below. We now state first Theorem \ref{T:StrongReg} in full details and then prove it.
\begin{theorem}[Linear convergence rate]
Assume that $\mR_{\mY}\left(\cdot\right)$ is $\delta$-strongly convex. Let $\left\{ \left(\X_{t} , \Y_{t}\right) \right\}_{t \geq 1}$ be a sequence produced by Algorithm \ref{alg:GCG} using the fixed step-size $\eta_{t} = \min\{ \alpha , \delta \}/\left(2\beta\right)$ for all $t \geq 1$. Then, for all $t\geq 1$, we have that 
\begin{equation*}
	f\left(\X_{t} , \Y_{t}\right) - f^{\ast} \leq \left(f\left(\X_{1} , \Y_{1}\right) - f^{\ast}\right)\cdot \exp\left(-\frac{\min\left\{ \alpha , \delta \right\}}{2\beta}\left(t - 1\right)\right).
\end{equation*}
\end{theorem}
\begin{proof}
Fix some iteration $t \geq 1$. Using the optimal choice of $\W_{t}$, from the strong convexity of $\W\ \rightarrow act{\W , \nabla g\left(\Z_{t}\right)} + mR_{\mY}\left(W\right)$, we have that
\begin{equation*}
	\act{\W_{t} , \nabla g\left(\Z_{t}\right)} + \mR_{\mY}\left(\W_{t}\right) \leq \act{\Y^{\ast} , \nabla g\left(\Z_{t}\right)} + \mR_{\mY}\left(\Y^{\ast}\right) - \frac{\delta}{2}\norm{\W_{t} - \Y^{\ast}}^{2}.
\end{equation*}
Plugging the above inequality into Proposition \ref{P:Tech}, we have that
\begin{align*}
	f\left(\X_{t + 1} , \Y_{t + 1}\right) & \leq \left(1 - \eta_{t}\right)f\left(\X_{t} , \Y_{t}\right) + \eta_{t}\left(g\left(\Z_{t}\right) + \mR_{\mX}\left(\X^{\ast}\right) + \mR_{\mY}\left(\W_{t}\right)\right)  \nonumber \\
	& + \eta_{t}\act{\X^{\ast} + \W_{t} - \Z_{t} , \nabla g\left(\Z_{t}\right)} + \eta_{t}^{2}\beta\left(\norm{\Z_{t} - \Z^{\ast}}^{2} +\norm{\W_{t} - \Y^{\ast}}^{2}\right) \\
	& \leq \left(1 - \eta_{t}\right)f\left(\X_{t} , \Y_{t}\right) + \eta_{t}\left(g\left(\Z_{t}\right) + \mR_{\mX}\left(\X^{\ast}\right) + \mR_{\mY}\left(\Y^{\ast}\right)\right)  \nonumber \\
	& + \eta_{t}\act{\Z^{\ast} - \Z_{t} , \nabla g\left(\Z_{t}\right)} + \eta_{t}^{2}\beta\norm{\Z_{t} - \Z^{\ast}}^{2} - \eta_{t}\norm{\W_{t} - \Y^{\ast}}^{2}\left(\frac{\delta}{2} - \eta_{t}\beta\right) \\
	& \underset{(a)}{\leq} \left(1 - \eta_{t}\right)f\left(\X_{t} , \Y_{t}\right) + \eta_{t}\left(g\left(\Z^{\ast}\right) + \mR_{\mX}\left(\X^{\ast}\right) + \mR_{\mY}\left(\Y^{\ast}\right)\right)  \nonumber \\
	& - \eta_{t}\norm{\Z_{t} - \Z^{\ast}}^{2}\left(\frac{\alpha}{2} - \eta_{t}\beta\right) - \eta_{t}\norm{\W_{t} - \Y^{\ast}}^{2}\left(\frac{\delta}{2} - \eta_{t}\beta\right) \\
	& = \left(1 - \eta_{t}\right)f\left(\X_{t} , \Y_{t}\right) + \eta_{t}f\left(\X^{\ast} , \Y^{\ast}\right) \\
	& - \eta_{t}\norm{\Z_{t} - \Z^{\ast}}^{2}\left(\frac{\alpha}{2} - \eta_{t}\beta\right) - \eta_{t}\norm{\W_{t} - \Y^{\ast}}^{2}\left(\frac{\delta}{2} - \eta_{t}\beta\right),
\end{align*}
where (a) follows from the strong convexity of $g\left(\cdot\right)$, which implies that 
\begin{equation*}
	\act{\Z^{\ast} - \Z_{t} , \nabla g\left(\Z_{t}\right)} \leq g\left(\Z^{\ast}\right) - g\left(\Z_{t}\right) - \frac{\alpha}{2}\norm{\Z_{t} - \Z^{\ast}}^{2}.
\end{equation*}
Thus, subtracting $f\left(\X^{\ast} , \Y^{\ast}\right)$ from both sides and using the notation $h_{t} := f\left(\X_{t} , \Y_{t}\right) - f\left(\X^{\ast} , \Y^{\ast}\right)$, we conclude that
\begin{align*}
	\forall \,\, 0 < \eta_{t} \leq \frac{\min \left\{ \alpha , \delta \right\}}{2\beta} : \qquad h_{t +  1} \leq \left(1 - \eta_{t}\right)h_{t}.
\end{align*}
The theorem now follows from choosing $\eta_{t} = \min\{ \alpha , \delta \}/\left(2\beta\right)$ and using standard manipulations.
\end{proof}

\section{Numerical Results}
In this section we present evidence for the empirical performance of our Algorithm \ref{alg:GCG} on the Robust PCA problem in the constrained formulation given in Problem \eqref{eq:robustPCA}. Focusing on first-order methods that are scalable to high-dimensional problems involving optimization problems with a nuclear-norm constraint, we compare our method to other competing projection-free first-order methods that avoid using high-rank SVD computations.

We tested our Algorithm \ref{alg:GCG} and compare to two methods: the standard conditional gradient method and the conditional gradient variant proposed in \cite{mu2016scalable}, which adds an additional proximal step to update the sparse component $\Y$, on top of the standard CG method. See Table \ref{table:algs} for description of the algorithms.

\begin{table*}\renewcommand{\arraystretch}{1.3}
{\small
\begin{center}{\footnotesize
  \begin{tabular}{| l | p{9cm} |} \hline
abbv. & description \\ \hline
ALT-CGPG & Algortihm \ref{alg:GCG}, sparse component $\Y$ updated via CG, low-rank component $\X$ updated via low-rank PG (see Section \ref{sec:lowranksvd} for details) \\ \hline
ALT-PGCG & Algortihm \ref{alg:GCG}, sparse component $\Y$ updated via PG, low-rank component $\X$ updated via CG \\ \hline
CGCG & Both blocks $\X$ and $\Y$ updated via CG \\ \hline
CGCG-P & Algorithm FW-P from \cite{mu2016scalable} - both $\X,\Y$ updated via CG, followed by an additional PG update to sparse component $\Y$ only \\ \hline
  \end{tabular}
  \caption{Description of the tested algorithms. CG is a short notation for a conditional gradient update, and PG is a short notation for a proximal gradient update.}
  \label{table:algs}}
\end{center}}
\vskip -0.2in
\end{table*}\renewcommand{\arraystretch}{1}

For all algorithms, we apply a line-search procedure to compute the optimal convex combination taken in the conditional gradient-like step on each iteration. This implementation is straightforward for the standard conditional gradient method (CGCG) and the proposed variant of \cite{mu2016scalable} (CGCG-P). For our proposed methods (ALT-CGPG, ALT-PGCG), we set the step-size for the proximal gradient update on each iteration $t$ to $\eta_t = \frac{2}{t+1}$ (this seems as a very good and practical approximation to the choice in Theorem \ref{thm:minDiam:detail} once we neglect the first short phase with a constant step-size, and start immediately with the second regime of step-sizes). Once we have computed the proximal update with this step-size (Line 5 of Algorithm \ref{alg:GCG}), we use a line-search to set the optimal convex combination parameter (used in Line 6 of Algorithm \ref{alg:GCG}). It is straightforward to argue that performing such a line-search instead of using the pre-fixed value of $\eta_t$ used for the proximal update, does not hurt any of the guarantees specified in Theorems \ref{thm:minDiam}, \ref{T:StrongSet}, \ref{T:StrongReg}, but can be quite important from a practical point of view.

\subsection{Experiments}
We generate synthetic data as follows. We set the dimensions in all experiments to $m=n=1000$. We generate the low-rank component by taking $\L = 10\U\V^{\top}$, where $\U$ is $m \times r$ and $\V$ is $n \times r$, where $r$ varies, and the entries of $\U$ and $\V$ are i.i.d. standard Gaussian random variables. The sparse component $\S$ is generated by setting $\S = 10\N$, where $\N$ is a matrix with i.i.d. standard Gaussian entries, and each entry in $\S$ is set to zero with probability $1 - p$ (independently of all other entries), where $p$ varies. We then set the observed matrix to $\M = \L + \S$. For each value of $(r,p)$ we have generated 15 i.i.d. experiments. See Table \ref{table:tests} for a quick summary. 

\begin{table*}\renewcommand{\arraystretch}{1.3}
{\small
\begin{center}{\footnotesize
  \begin{tabular}{| c | c | c | c | c | c |} \hline
config. & fig. & rank of $\L$ ($r$)& sampling freq. in $\S$ ($p$)& avg. $\norm{\L}_{\ast}$ & avg. $\norm{\S}_{1}$ \\ \hline
1 & 1& 5  & 0.001& 4.9926e+04 & 7.9669e+04 \\ \hline
2 & 1 & 5  & 0.003 & 4.9682e+04 & 2.3806e+05 \\ \hline
3 & 1 & 25  &  0.001 & 2.4872e+05 & 7.9837e+04 \\ \hline
4 & 1 & 25  & 0.003 & 2.4853e+05 & 2.3675e+05 \\ \hline
5 & 2 &25 &  0.03 & 2.4830e+05 & 2.3914e+06 \\ \hline
6 & 2  & 130 & 0.01 & 1.2589e+06 & 7.9836e+05 \\ \hline
  \end{tabular}
  \caption{Description of data used. The forth and fifth columns give the values  $\norm{\L}_{\ast}$ and $\norm{\S}_{1}$ averaged over the $15$ i.i.d. experiments.}
  \label{table:tests}}
\end{center}}
\vskip -0.2in
\end{table*}\renewcommand{\arraystretch}{1}

All algorithms were used with the exact parameters $\tau = \norm{\L}_{\ast}$ and $s = \norm{\S}_{1}$, and with the same initialization point $(\X_1,\Y_1)=(\mathbf{0},\mathbf{0})$. The algorithms were implemented in Matlab with the \texttt{svds} command used to compute the low-rank SVD updates. For our algorithm ALT-CGPG we have used a rank-$r$ SVD to compute the low-rank proximal update (were $r$ is the rank of the low-rank data component $\L$). For all experiments we measured the function value (as given in Problem \eqref{eq:robustPCA}) both as a function of the number of iterations executed and and the overall runtime.

From the results in Figure \ref{fig:5} and Figure \ref{fig:6} it is clear that in each one of the six scenarios tested (different values of $r$ and $p$), at least one of the variants of our Algorithm \ref{alg:GCG} (either ALT-CGPG or ALT-PGCG) clearly outperforms CGCG and CGCG-P. In particular, when examining the results and the corresponding norm bounds $\tau$ and $s$ (as recorded in Table \ref{table:tests}), we can see evidence for the improved complexity achieved in Theorem \ref{thm:minDiam}, which presents improved convergence bound in  low/high Signal-to-Noise regimes (see also discussion in Section \ref{sec:SNR}). We can see that, as a rule of thumb, indeed when $\norm{\S}_{1} >> \norm{\L}_{\ast}$, updating the sparse component $\Y$ via a proximal-gradient update in Algorithm \ref{alg:GCG}, results in significantly faster performances than when a conditional-gradient update is used. Similarly, when $\norm{\S}_{1} >> \norm{\L}_{\ast}$, we can see that updating the low-rank component $\X$ via a (low-rank) proximal-update in Algorithm \ref{alg:GCG}, results in significantly faster performances of our algorithm. Perhaps surprisingly, it also seems that the standard conditional gradient method (CGCG) outperforms the variant recently proposed in \cite{mu2016scalable} (CGCG-P), with configuration 5 ($r=25$, $p=0.03$) being the exception.

\begin{figure}
  \begin{subfigure}[t]{0.5\textwidth}
    \centering
    \includegraphics[width=\textwidth]{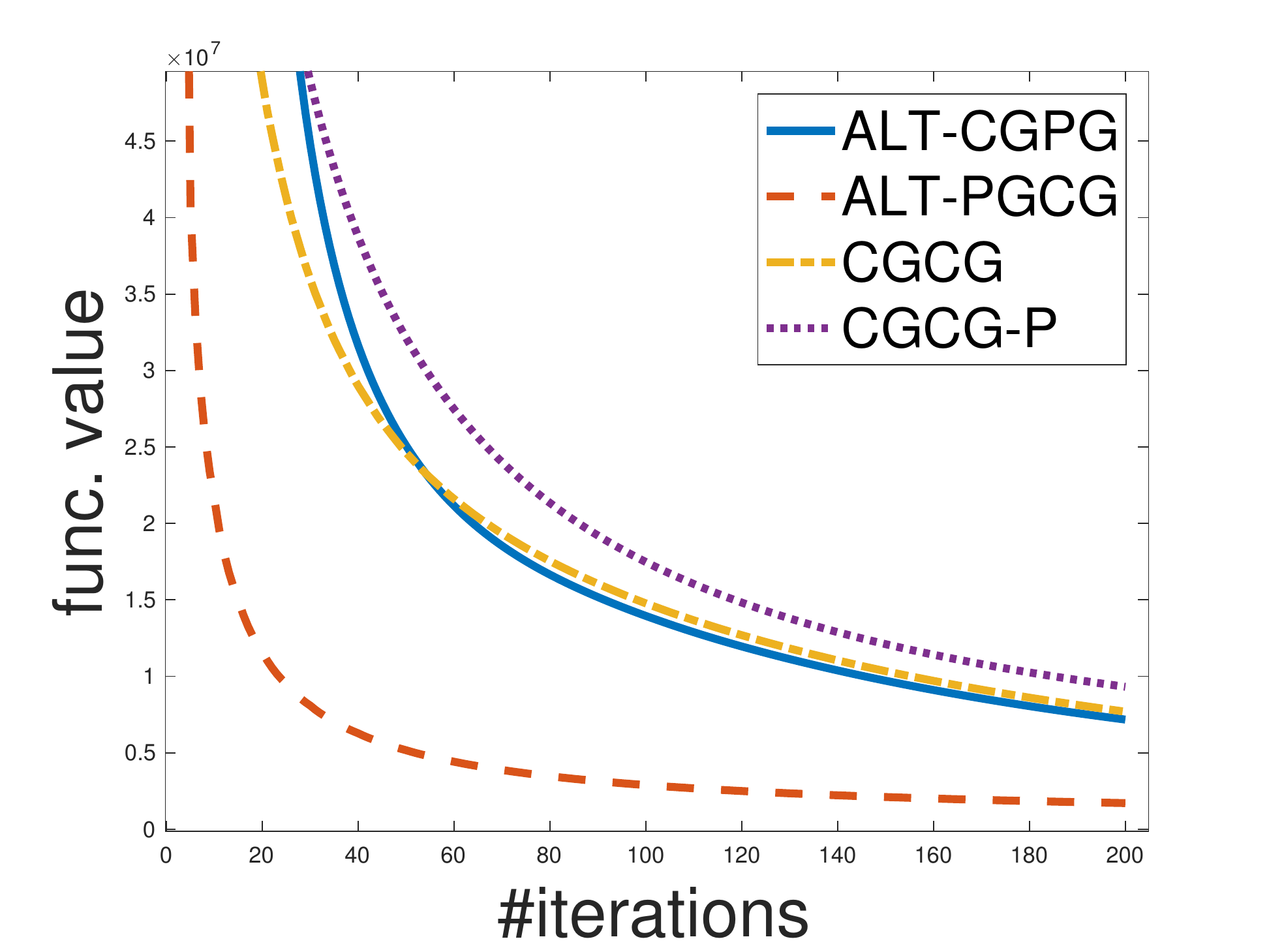}
     \caption{$r=5$ and $p = 0.001$}
  \end{subfigure}
  ~
  \begin{subfigure}[t]{0.5\textwidth}
    \centering
    \includegraphics[width=\textwidth]{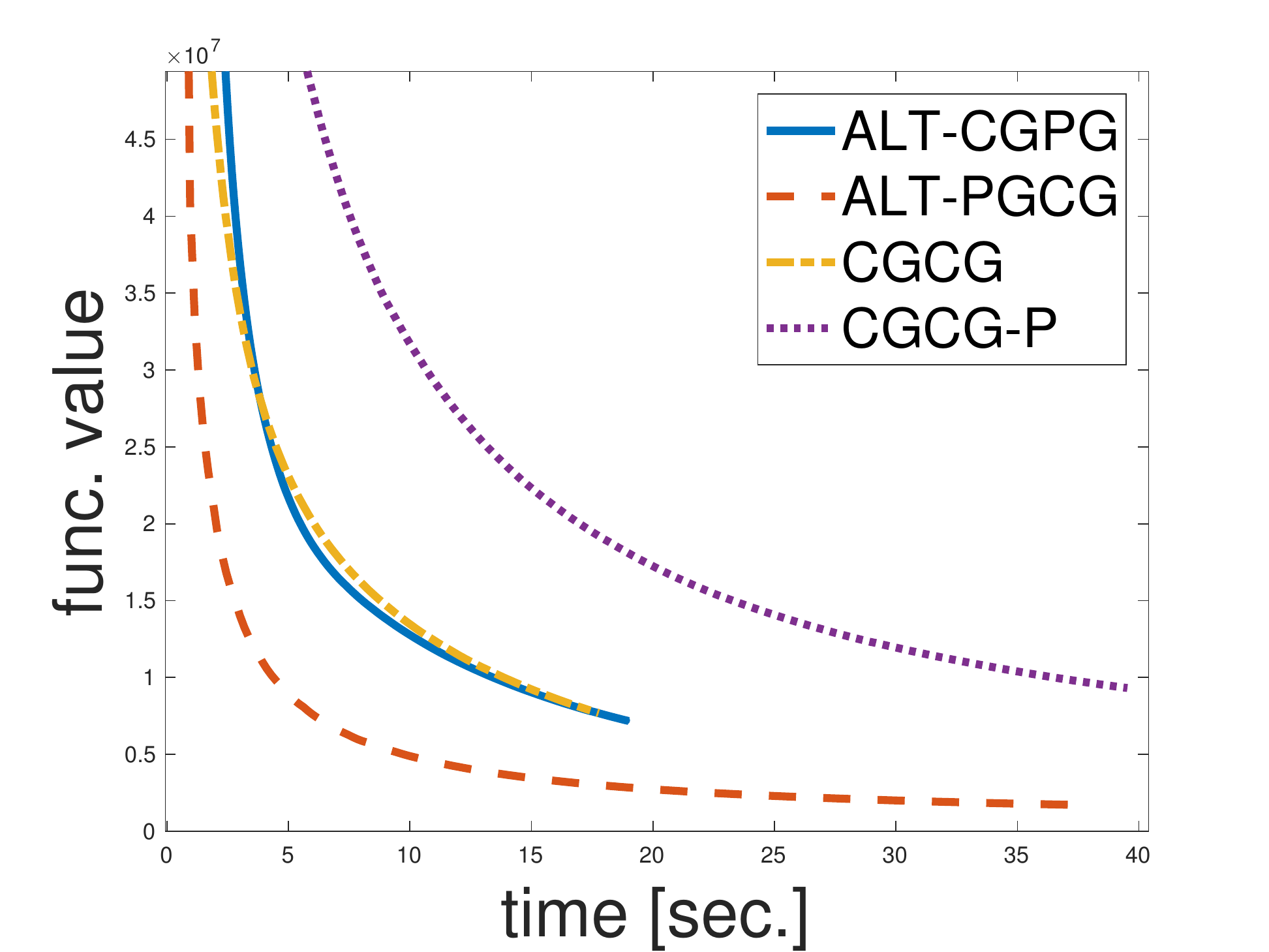}
   \caption{$r=5$ and $p = 0.001$}
  \end{subfigure}
  
  \begin{subfigure}[t]{0.5\textwidth}
    \centering
    \includegraphics[width=\textwidth]{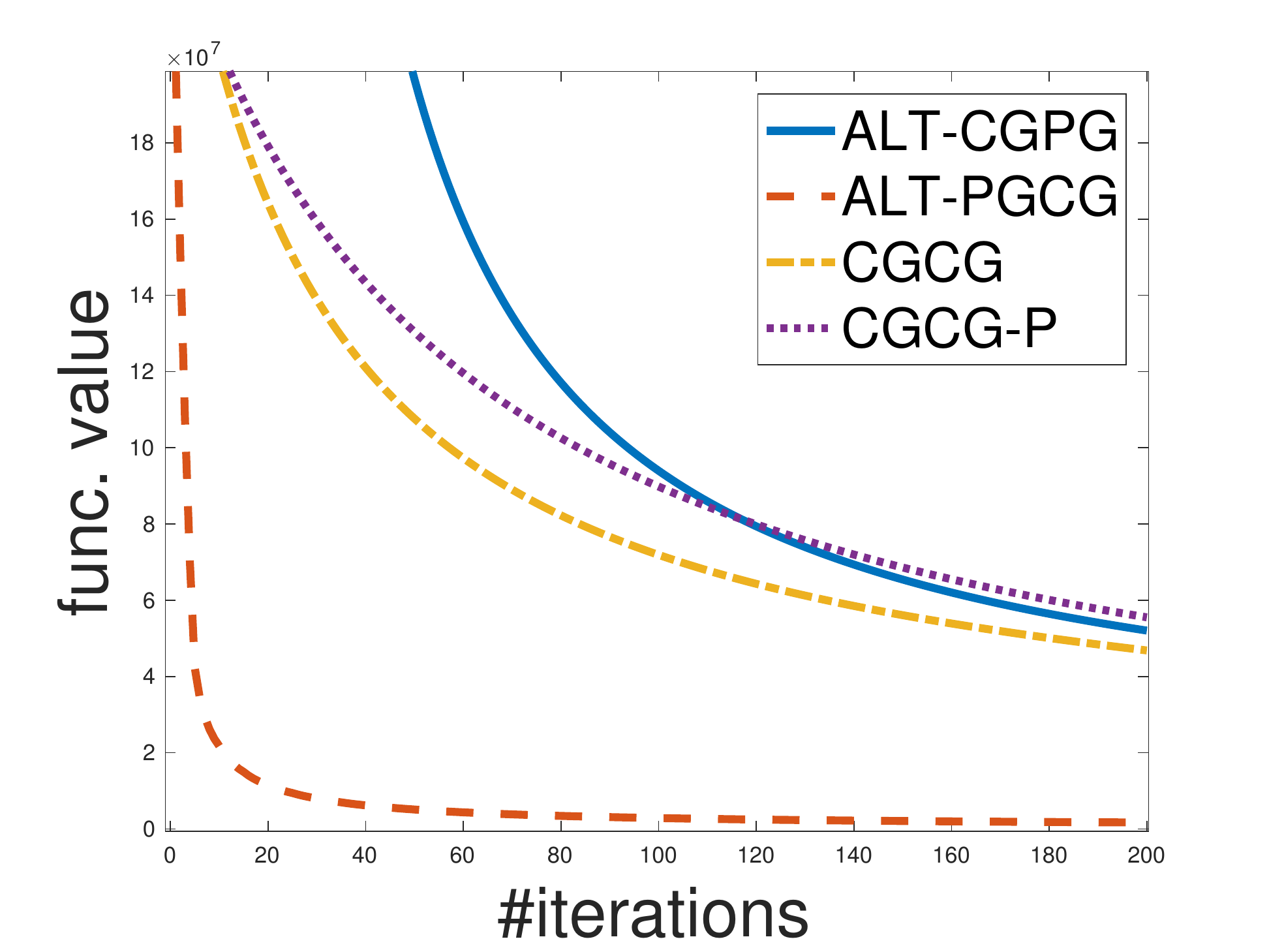}
 \caption{$r=5$ and $p = 0.003$}
  \end{subfigure}
  ~
  \begin{subfigure}[t]{0.5\textwidth}
    \centering
    \includegraphics[width=\textwidth]{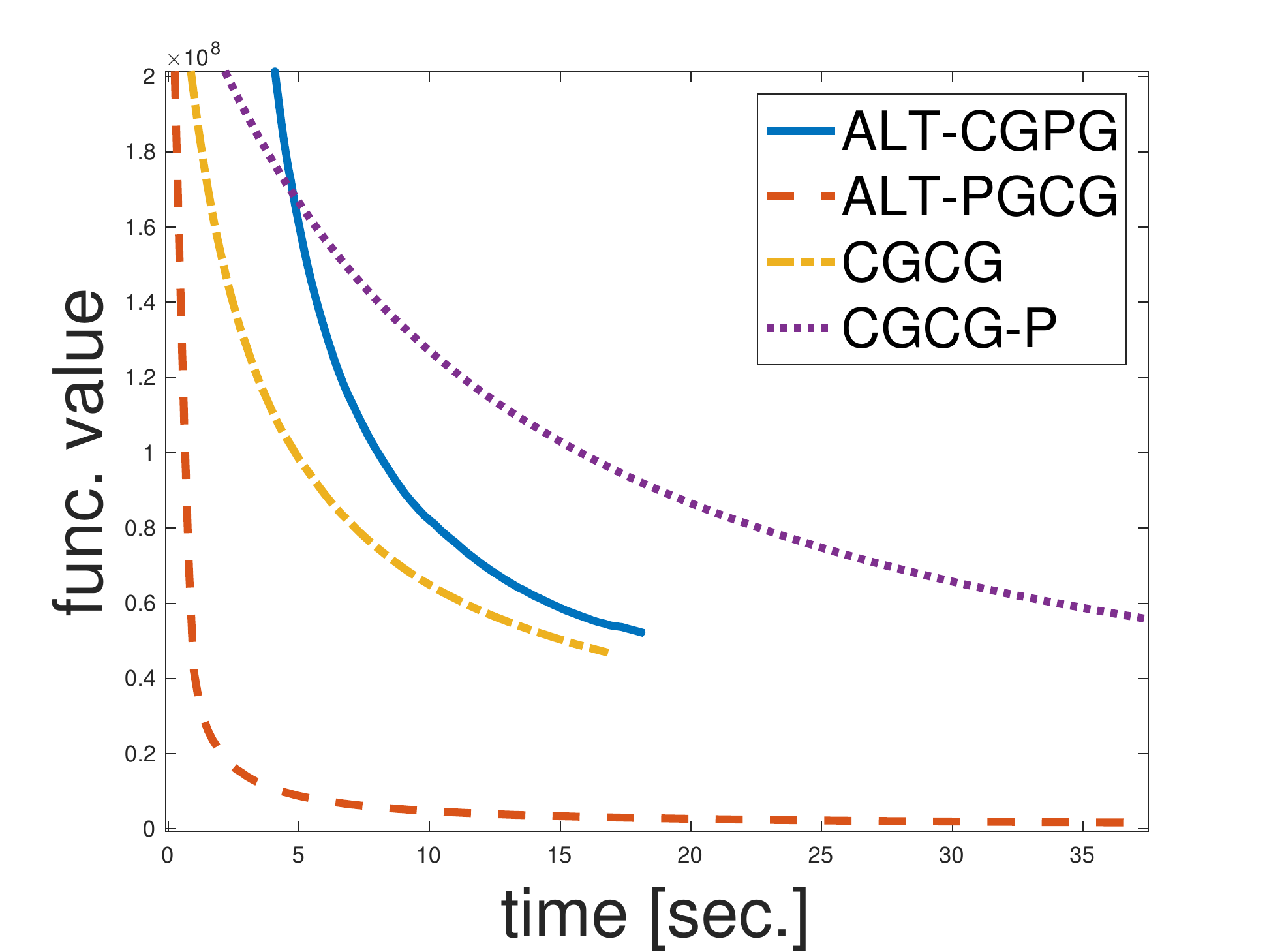}
 \caption{$r=5$ and $p = 0.003$}
  \end{subfigure}

  \begin{subfigure}[t]{0.5\textwidth}
    \centering
    \includegraphics[width=\textwidth]{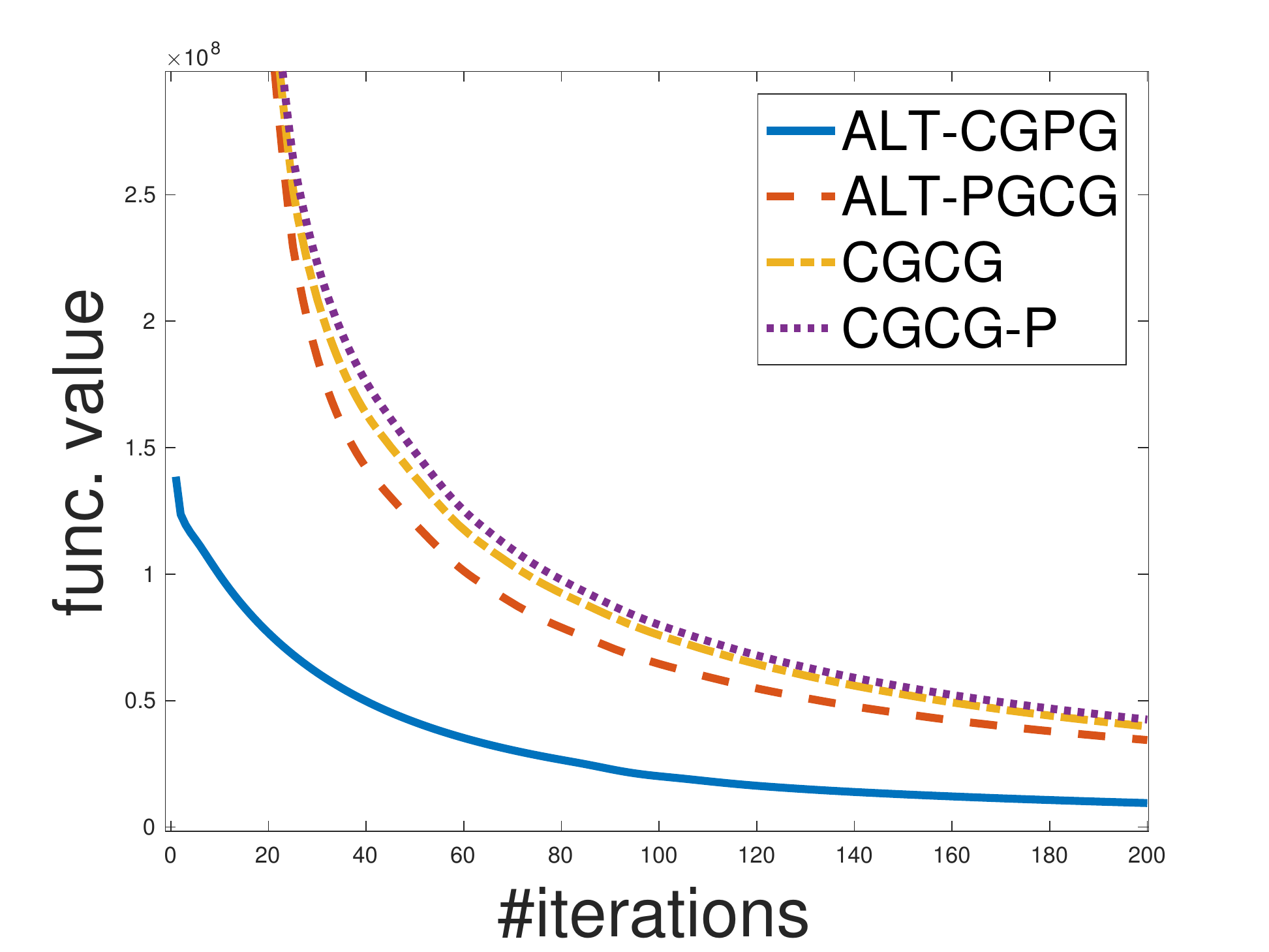}
  \caption{$r=25$ and $p = 0.001$}
  \end{subfigure}
  ~
  \begin{subfigure}[t]{0.5\textwidth}
    \centering
    \includegraphics[width=\textwidth]{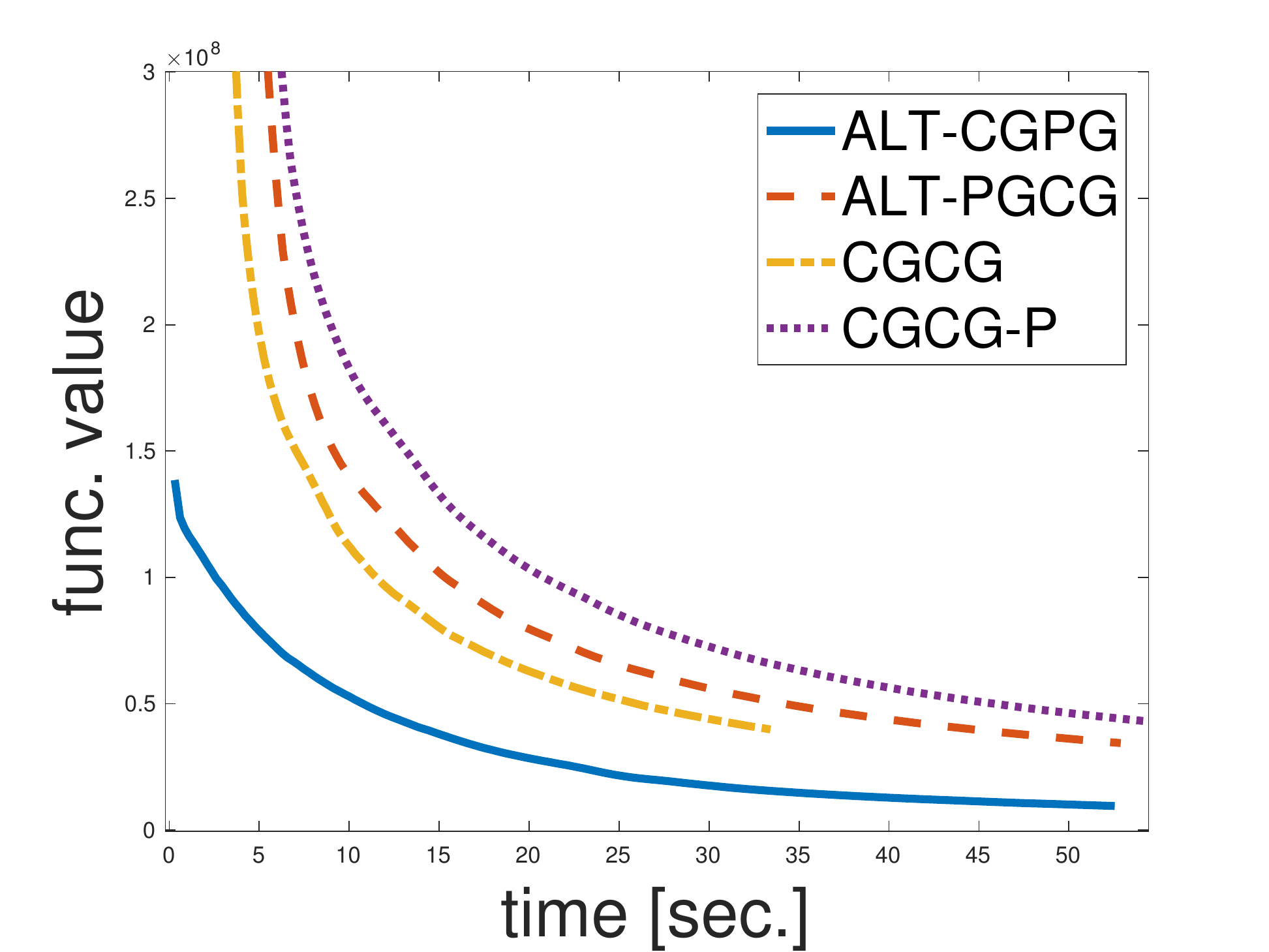}
  \caption{$r=25$ and $p = 0.001$}
  \end{subfigure}

  \begin{subfigure}[t]{0.5\textwidth}
    \centering
    \includegraphics[width=\textwidth]{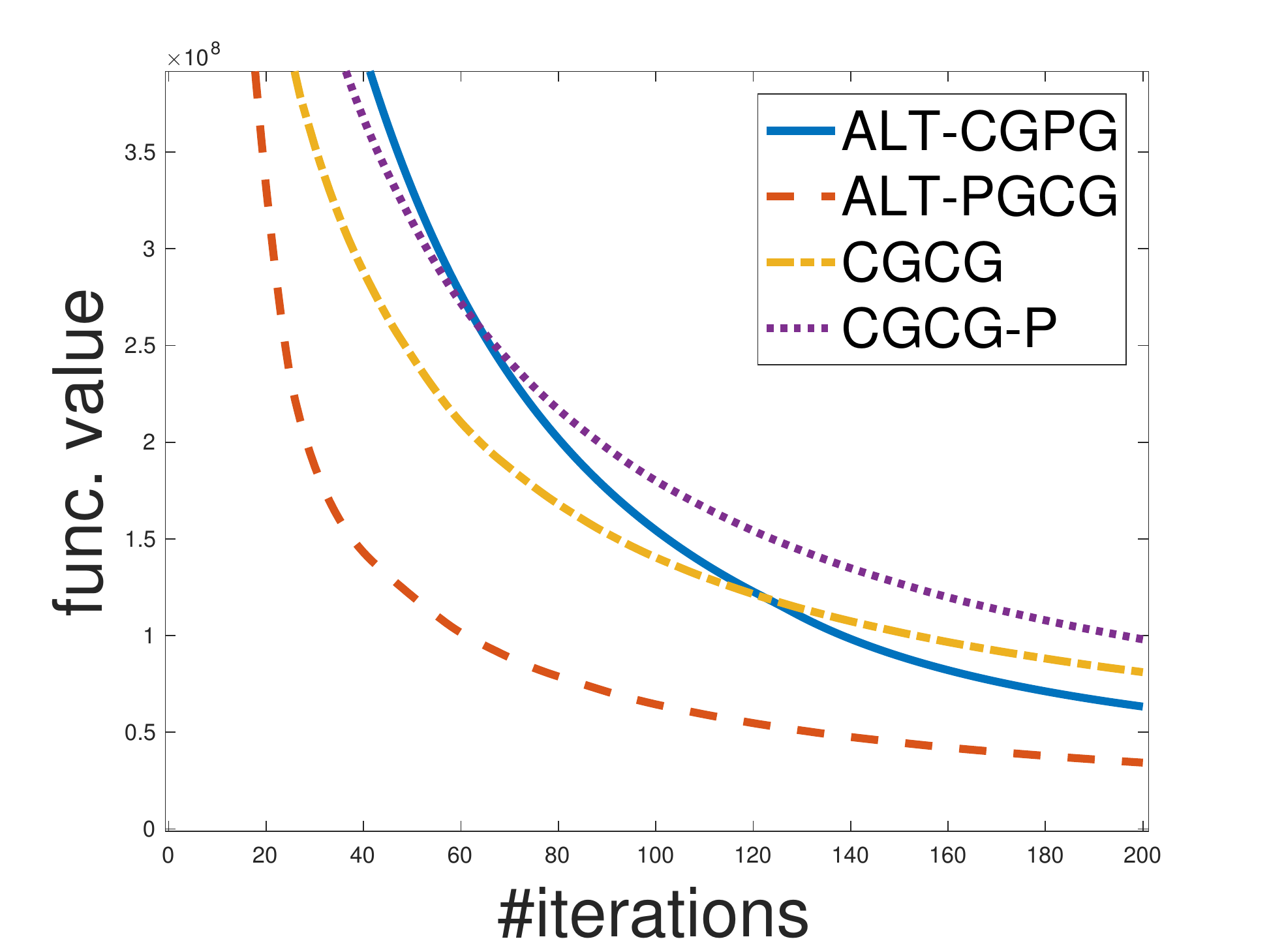}
      \caption{$r=25$ and $p = 0.003$}
  \end{subfigure}
  ~
  \begin{subfigure}[t]{0.5\textwidth}
    \centering
    \includegraphics[width=\textwidth]{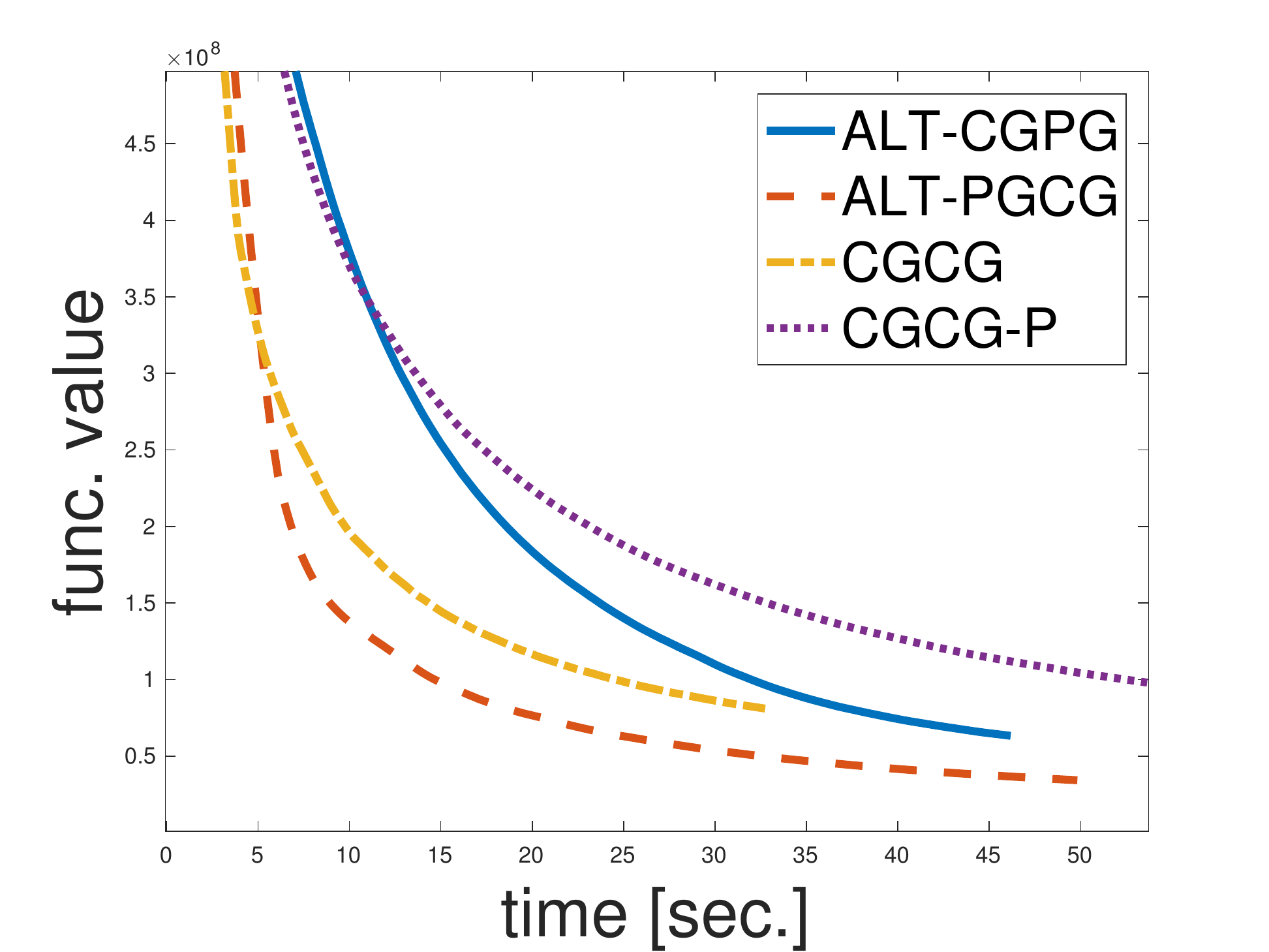}
          \caption{$r=25$ and $p = 0.003$}
  \end{subfigure}
  \caption{Results for configurations 1-4. Each configuration presents the average over  15 i.i.d. runs.}
  \label{fig:5}
\end{figure}

\begin{figure}
  \begin{subfigure}[t]{0.5\textwidth}
    \centering
    \includegraphics[width=\textwidth]{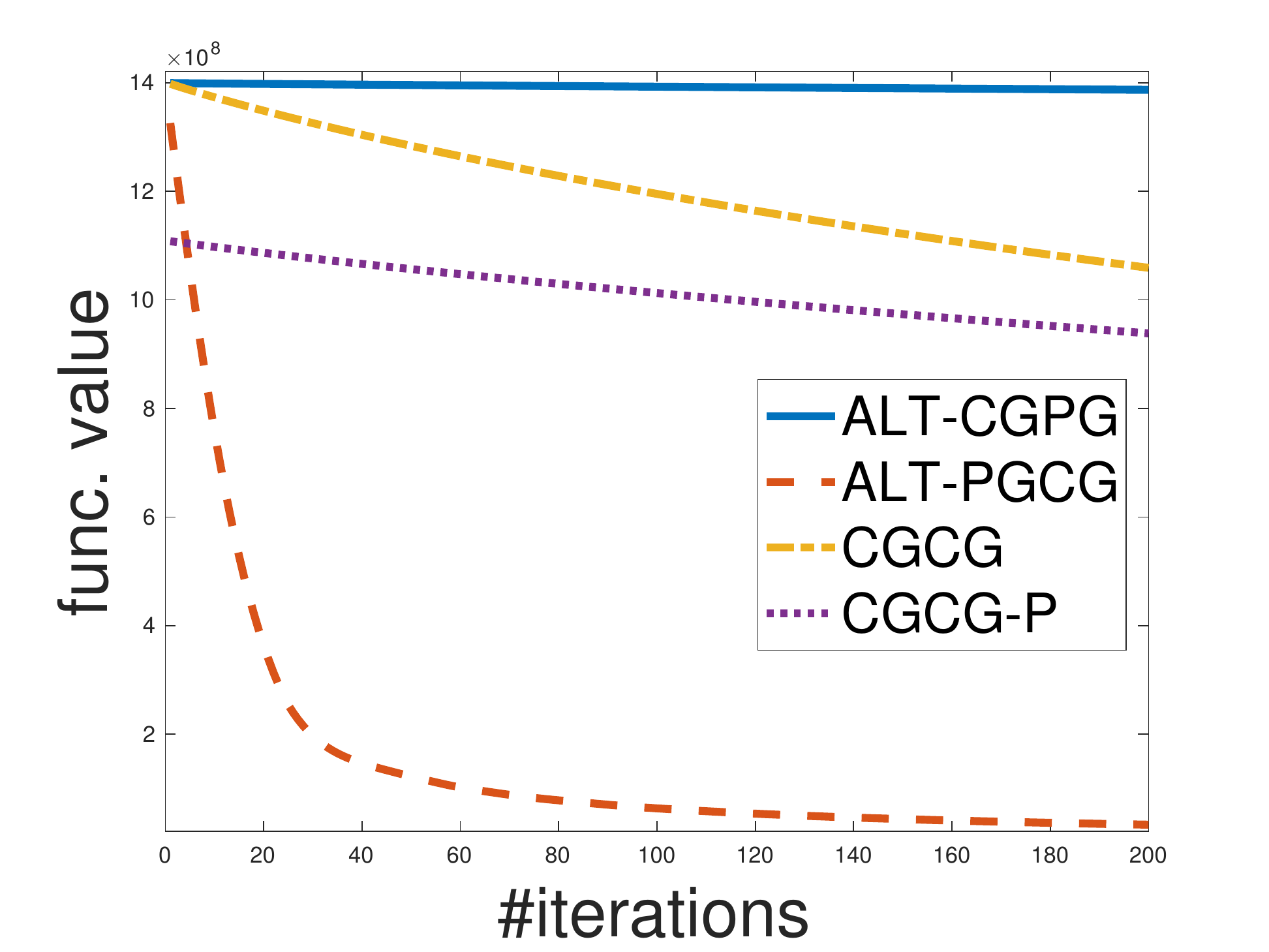}
  \caption{$r=25$ and $p = 0.03$}
  \end{subfigure}
  ~
  \begin{subfigure}[t]{0.5\textwidth}
    \centering
    \includegraphics[width=\textwidth]{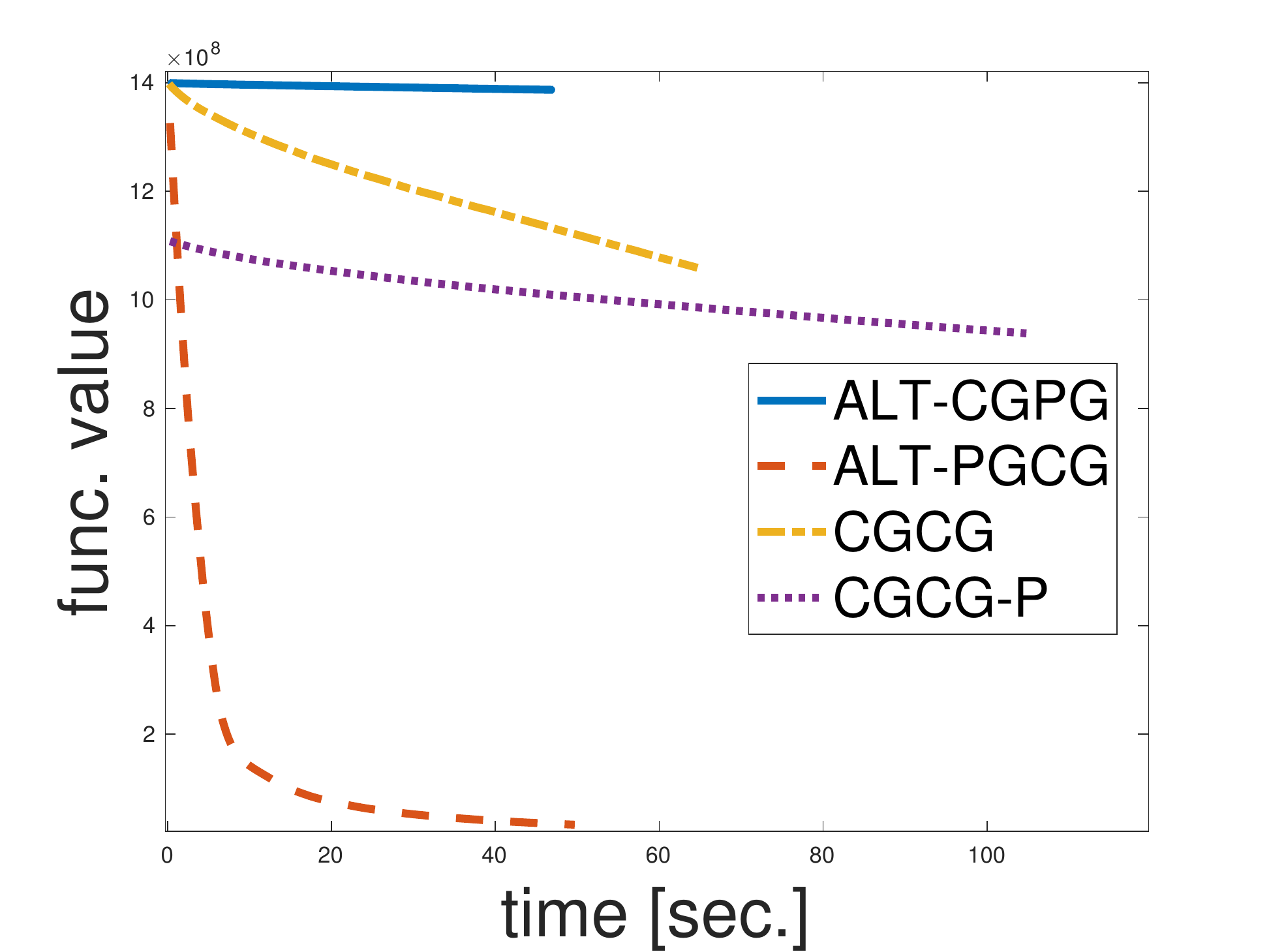}
  \caption{$r=25$ and $p = 0.03$}
  \end{subfigure}

  \begin{subfigure}[t]{0.5\textwidth}
    \centering
    \includegraphics[width=\textwidth]{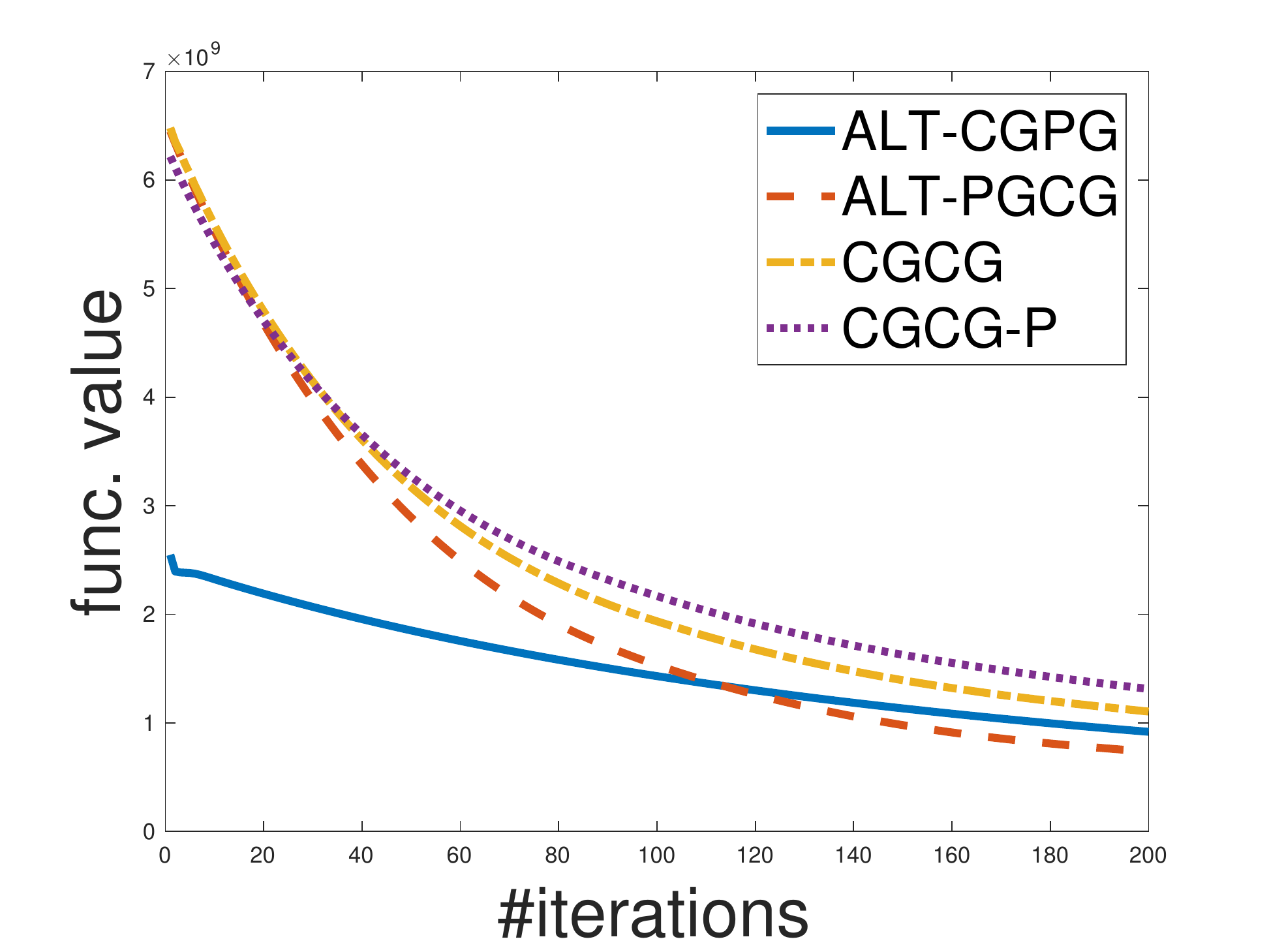}
      \caption{$r=130$ and $p = 0.01$}
  \end{subfigure}
  ~
  \begin{subfigure}[t]{0.5\textwidth}
    \centering
    \includegraphics[width=\textwidth]{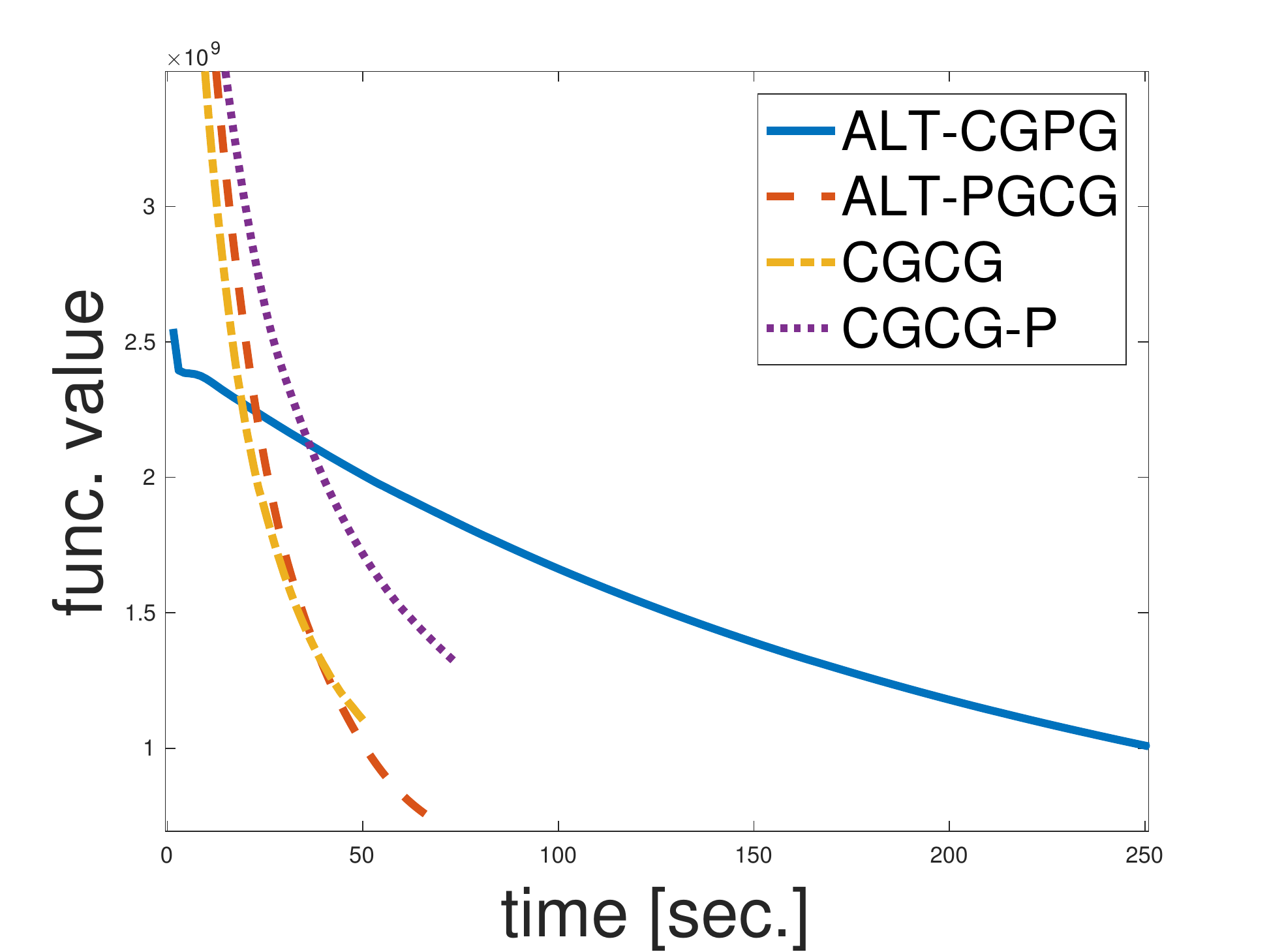}
          \caption{$r=130$ and $p = 0.01$}
  \end{subfigure}
  \caption{Results for configurations 5-6. Each configuration presents the average over  15 i.i.d. runs.}
  \label{fig:6}
\end{figure}

\paragraph{Acknowledgments} Dan Garber is supported by the ISRAEL SCIENCE FOUNDATION (grant No. 1108/18).

\section{Appendix}
\begin{lemma} \label{L:Reco}
Consider a sequence $\left\{ h_{t} \right\}_{t \geq 1} \subset \reals_{+}$ satisfying the recursion:
\begin{equation*}
	\forall \,\, t \geq 1 \, \forall \eta_{t} \in \left(0 , \frac{\alpha}{2\beta}\right] : \qquad h_{t + 1} \leq \left(1 - \eta_{t}\right)h_{t} + \eta_{t}^{2}\beta D_{\mY}^{2}.
\end{equation*}
Then, setting the step-size $\eta_{t}$ according to:
\begin{equation*}
	\eta_{t} = 
		\begin{cases}
        		\frac{\alpha}{2\beta} & \mbox{if }t \leq t_{0}, \\
        		\frac{2}{t - t_{0} + \frac{4\beta}{\alpha}} & \mbox{if }t > t_{0},
        	\end{cases}
\end{equation*} 
where $t_{0} := \max \left\{ 0 , \left\lceil{ 2\beta/\left(\alpha\right)\ln\left(2C/\left(\alpha D_{\mY}^{2}\right)\right)}\right\rceil \right\}$, for $C$ satisfying $C \geq h_{1}$, guarantees, for all $t \geq t_{0} + 1$ that
\begin{equation*}
	h_{t} \leq \frac{4\beta D_{\mY}^{2}}{t - t_{0} - 1 + \frac{4\beta}{\alpha}}.
\end{equation*}
\end{lemma}
\begin{proof}
Let us define $v_{t} := h_{t}/\left(2\beta D_{\mY}^{2}\right)$ for all $t \geq 1$. Dividing both sides of the recursion in the lemma by $2\beta D_{\mY}^{2}$, we obtain the recursion
\begin{equation}\label{eq:mindiam:1}
	\forall \,\, 0 < \eta_{t} \leq \frac{\alpha}{2\beta} : \qquad v_{t + 1} \leq \left(1 - \eta_{t}\right)v_{t} + \frac{\eta_{t}^{2}}{2}.
\end{equation}
Let $C , t_{0}$ and $\left\{ \eta_{t} \right\}_{t \geq 1}$ be as defined in the lemma, and recall that $\eta_{t} = \alpha/\left(2\beta\right)$, for all $t \leq t_{0}$. Using Eq. \eqref{eq:mindiam:1}, we have that
\begin{align*}
	v_{t_{0} + 1} & \leq \left(1 - \eta_{0}\right)^{t_{0}}v_{1} + \frac{\eta_{0}^{2}}{2}\sum_{t = 1}^{t_{0}} \left(1 - \eta_{0}\right)^{t - 1} \leq \frac{C}{2\beta D_{\mY}^{2}} \cdot e^{-\eta_{0}t_{0}} + \frac{\eta_{0}^{2}}{2} \cdot \frac{1}{\eta_{0}} \\
	& = \frac{C}{2\beta D_{\mY}^{2}} \cdot e^{-\frac{\alpha t_{0}}{2\beta}} + \frac{\alpha}{4\beta}.
\end{align*}
Thus, for $t_{0} := \max \left\{ 0 , \left\lceil{ 2\beta/\left(\alpha\right)\ln\left(2C/\left(\alpha D_{\mY}^{2}\right)\right)}\right\rceil \right\}$, we obtain that $v_{t_{0} + 1} \leq \alpha/\left(2\beta\right)$.

We now show that for all $t \geq t_{0} + 1$, it holds that $v_{t} \leq 2/\left(t - t_{0} + 1 + c_{0}\right)$ for $c_{0} := 4\beta/\alpha - 2$. For the base case $t = t_{0} + 1$, we indeed have already showed that $v_{t_{0} + 1} \leq \alpha/\left(2\beta\right)$, as needed. Note that using the step-size choice $\eta_{t} :=  2/\left(t - t_{0} + 2 + c_{0}\right)$ for all $t \geq t_{0} + 1$, as defined in the lemma, we have that $\eta_{t} \leq \eta_{t_{0} + 1} = 2/\left(3 + c_{0}\right) < 2/c_{0} = \alpha/\left(2\beta\right)$, and hence we can apply the recursion \eqref{eq:mindiam:1} for all $t \geq t_{0} + 1$. Thus, assuming the induction holds for some $t \geq t_0+1$, using the recursion \eqref{eq:mindiam:1}, the induction hypothesis, and our step-size choice, we have that
\begin{align*}
	v_{t + 1} & \leq v_{t}\left(1 - \eta_{t}\right) + \frac{\eta_{t}^{2}}{2} \leq \frac{2}{t - t_{0} + 1 + c_{0}}\left({1 - \frac{2}{t - t_{0} + 2 + c_{0}}}\right) \\
	& + \frac{2}{\left(t - t_{0} + 2 + c_{0}\right)^{2}} \\
	& = \frac{2}{t - t_{0} + 2 + c_{0}}\left({1 + \frac{1}{t - t_{0} + 1 + c_{0}}}\right)\left({1 - \frac{2}{t - t_{0} + 2 + c_{0}}}\right) \\
	& + \frac{2}{\left(t - t_{0} + 2 + c_{0}\right)^{2}} \\
	& = \frac{2}{t - t_{0} + 2 + c_{0}}\left(1 + \frac{\left(t - t_{0} + 2 + c_{0}\right) - 2\left(t -t_{0} + 1 + c_{0}\right) - 2 + \left(t - t_{0} + 1 + c_{0}\right)}{\left(t - t_{0} + 1 + c_{0}\right)\left(t - t_{0} + 2 + c_{0}\right)}\right) \\
 	& = \frac{2}{t - t_{0} + 2 + c_{0}}\left(1 - \frac{1}{\left(t - t_{0} + 1 + c_{0}\right)\left(t -t_{0} + 2 + c_{0}\right)}\right) \\
	& < \frac{2}{t - t_{0} + 2 + c_{0}}.
\end{align*}
Hence, the induction implies that $v_{t} \leq 2/\left(t - t_{0} - 1 + 4\beta/\alpha\right)$ for all $t \geq t_{0} + 1$. The proof is completed by recalling that $h_{t} = 2\beta D_{\mY}^{2}v_{t}$.
\end{proof}
\begin{lemma}\label{lem:strongset:recurs}
Consider a sequence $\left\{ h_{t} \right\}_{t \geq 1} \subset \reals_{+}$ satisfying the recursion:
\begin{equation*}
	\forall t \geq 1 : \qquad h_{t+1} \leq h_{t} - \eta_{t}h_{t},
\end{equation*}
where $0 < \eta_{t} \leq \min\{ c_{1}\sqrt{h_{t}} , c_{2} \}$, $c_{1} > 0$ and $0 < c_{2} \leq 1$. Then, setting $\eta_{t} = 3/\left(t - 1 + 3c_{2}^{-1}\right)$ for all $t \geq 1$, yields that $h_{t} \leq 9\max\{ c_{1}^{-2} , c_{2}^{-2}h_{1} \}/\left(t - 1 + 3c_{2}^{-1}\right)^{2}$.
\end{lemma}
\begin{proof}
We prove via induction that for suitably chosen positive constants $a$ and $b$ that $h_{t} \leq a/\left(t + b\right)^{2}$ for all $t \geq 1$.

Fix some iteration $t \geq 1$ and suppose the claim holds for $h_{t}$. We consider now two cases. First, if $h_{t} \leq a/\left(t + b + 1\right)^{2}$, then, since by the recursion in the lemma it holds that $h_{t + 1} \leq h_{t}$, the claim clearly holds in this case for $h_{t + 1}$.

For the second case, we assume $h_{t} \geq a/\left(t + b + 1\right)^{2}$. Using this assumption, the recursion, and the induction hypothesis, we have that 
\begin{align*}
	h_{t + 1} & \leq h_{t} - \eta_{t}h_{t} \\
	& \leq \frac{a}{\left(t + b\right)^{2}} - \eta_{t}\frac{a}{\left(t + b + 1\right)^{2}} \\
	& = \frac{a}{\left(t + b + 1\right)^{2}}\left(\left(\frac{t + b + 1}{t + b}\right)^{2} - \eta_{t}\right) \\
	& \leq \frac{a}{\left(t + b + 1\right)^{2}}\left({1 + \frac{3}{t + b} - \eta_{t}}\right).
\end{align*}
Thus, for $\eta_{t} = 3/\left(t + b\right)$, the induction clearly holds.

Since, for the recursion to hold, it is also required that $\eta_{t} \leq \min\{ c_{1}\sqrt{h_{t}} , c_{2} \} \leq \min\{ c_{1}\sqrt{a}/\left(t + b\right) , c_{2} \}$, this brings us to the following conditions on $a$ and $b$ which should be valid for all $t \geq 1$:
\begin{equation*}
	\frac{3}{t + b} \leq c_{2} \qquad \mbox{and} \qquad \frac{3}{t + b} \leq \frac{c_{1}\sqrt{a}}{t + b}.
\end{equation*}
Thus, we get the requirements $b \geq 3c_{2}^{-1} - 1$ and $a \geq 9c_{1}^{-2}$.

Finally, since for the base case $t = 1$ it needs to hold that $h_{1} \leq a/\left(b + 1\right)^{2}$, we can choose $b = 3c_{2}^{-1} - 1$ and $a = 9\max\{ c_{1}^{-2} , c_{2}^{-2}h_{1} \}$, which guarantee that.
\end{proof}

\bibliographystyle{plain}
\bibliography{bib}

\end{document}